\documentclass{article}
\usepackage{times}
\usepackage{soul}
\usepackage{url}
\usepackage[small]{caption}
\usepackage{graphicx}

\usepackage{amsmath}
\usepackage{amssymb}
\usepackage{mathtools}
\usepackage{amsthm} 
\usepackage{booktabs}
\usepackage{algorithm}
\usepackage{algorithmic}
\usepackage[switch]{lineno}

\usepackage{pgfplots}
\usepackage{subfigure}
\usepackage{tikz}
\pgfplotsset{compat=1.18}  
\usepgfplotslibrary{units}  

\usepackage{xcolor}
\usepackage{hyperref}
\hypersetup{hidelinks}
\usepackage[capitalize,noabbrev]{cleveref}

\usepackage{booktabs}

\usepackage{amsfonts}
\usepackage{multirow}
\usepackage[round,sort&compress]{natbib}
\setcitestyle{authoryear,round}

\hypersetup{
    colorlinks = true,
    linkcolor = blue,
    filecolor = magenta,
    urlcolor = cyan,
}

\newtheorem{theorem}{Theorem}

\newtheorem*{claim*}{Claim}

\newtheorem{lemma}[theorem]{Lemma}

\newtheorem{corollary}[theorem]{Corollary}
\theoremstyle{definition}

\newtheorem*{remark*}{Remark}

\numberwithin{theorem}{section} 

\title{Transformers in Pseudo-Random Number Generation: \\
A Dual Perspective on Theory and Practice}
\author{
    Ran Li\thanks{School of Information Science and Technology, 
    Northeast Normal University, Changchun, China. 
    Email: \href{mailto:lir660@nenu.edu.cn}{lir660@nenu.edu.cn},
    \href{mailto:zengls188@nenu.edu.cn}{zengls188@nenu.edu.cn}}
    \thanks{Equal Contribution}
    \and 
    Lingshu Zeng \footnotemark[1] \footnotemark[2]
}
\date{}

\begin{document}

\maketitle

\begin{abstract}
    Pseudo-random number generators (PRNGs) are high-nonlinear processes, and they are key blocks in optimization of Large language models. Transformers excel at processing complex nonlinear relationships. Thus it is reasonable to generate high-quality pseudo-random numbers based on transformers. In this paper, we explore this question from both theoretical and practical perspectives, highlighting the potential benefits and implications of Transformer in PRNGs. We theoretically demonstrate that decoder-only Transformer models with Chain-of-Thought can simulate both the Linear Congruential Generator (LCG) and Mersenne Twister (MT) PRNGs. Based on this, we conclude that the log-precision decoder-only Transformer can represent non-uniform $\text{AC}^0$. Our simulative theoretical findings are validated through experiments. The random numbers generated by Transformer-based PRNGs successfully pass the majority of NIST tests, whose heat maps exhibit clear statistical randomness. Finally, we assess their capability in prediction attacks.
\end{abstract}

\section{Introduction}
\label{sec:introduction}

Pseudo-random numbers are essential components in scientific and technological applications, from cryptography to statistical sampling and numerical computing. Modern applications require PRNGs that can generate high-quality random sequences. A PRNG uses a deterministic algorithm to produce reproducible sequences that appear random, based on an initial seed value.~\cite{gentle2003random,press2007numerical}.

In recent years, Transformer-based large language models (LLMs) have shown exceptional performance across multiple domains. While these models excel in language processing and knowledge tasks, their theoretical foundations remain unclear. Specifically, how Transformers simulate complex nonlinear functions is still not fully understood.


Recent theoretical advances have provided significant insights into the expressive power of Transformers. Notably, it has been demonstrated that Transformers with constant depth and precision, when enhanced with Chain-of-Thought (CoT) prompting~\cite{wei2022chain} of polynomial length \( O(\text{poly}(n)) \), possess sufficient expressive power to capture the complexity class P/poly (Theorem 3.3 in~\cite{li2024chain}). This theoretical foundation suggests that Transformers might be particularly well-suited for processing nonlinear functions, with PRNGs serving as a classic and practical example of such functions.

The application of Transformers as PRNGs offers several key characteristics: (1) Transformer's advanced feature learning capabilities enable effective extraction and learning of complex statistical patterns in random sequences, (2) The sophisticated position encoding mechanisms systematically capture temporal dependencies and sequential relationships between numerical elements~\cite{wu2025pseudorandom}, (3) The multi-head attention architecture facilitates parallel computation across diverse representation subspaces, enabling simultaneous analysis of multiple sequence characteristics. Building on these advantages, our research explores a fundamental question: What are the advantages of Transformers as PRNGs or tools for evaluating PRNGs security? To address this, we conduct a comprehensive study combining theoretical analysis with empirical validation of Transformer models' capabilities in pseudo-random number generation.


\paragraph{Our result}
This paper makes several significant contributions to understand the
expressiveness of Transformer models in the context of PRNGs. Our work encompasses both theoretical analysis and empirical validation. Specifically, we demonstrate that Transformers can be utilized for prediction attacks, where in a Transformer model is trained on a sequence of generated pseudo-random numbers to predict subsequent values, thereby providing a mechanism to assess the security of the PRNGs.

First, we provide a comprehensive theoretical analysis of Transformer models'
capability to express two widely-used PRNGs: the Linear Congruential Generator
(LCG) and the Mersenne Twister (MT) algorithm. We demonstrate that Transformer architectures equipped with CoT
\footnote{In this paper, we define Chain of Thought as a step-by-step reasoning process, aligning with the interpretation in ~\cite{feng2024towards}, rather than as a specialized prompting technique.}
can precisely simulate these PRNGs using only a constant number of attention heads and layers, while maintaining $O(n)$ parameter complexity
in the hidden layer (~\Cref{thm:LCG} and~\Cref{thm:MT}). Our proofs are constructive in nature, providing explicit mechanisms for
simulating each computational step of both algorithms. To validate these
theoretical findings, we conduct comprehensive experiments across multiple
bit-widths (8-bit, 12-bit, and 16-bit) for pseudo-random number generation
based on the MT algorithm, achieving near-perfect accuracy in all cases.
Significantly, during this analysis, we establish a non-uniform $\text{AC}^0$ lower bound
for Decoder-only Transformers with $\log(n)$ precision (~\Cref{cor:transformer_ac0}),
contributing to the theoretical understanding of Transformer limitations.

Second, we perform a comprehensive empirical evaluation of Transformer-based
pseudo-random number generators (PRNGs) through three critical experiments:
\begin{itemize}
	\item We investigate the potential of Transformer architectures as novel implementations for PRNGs. Through empirical analysis utilizing heat map visualizations, we clearly observe that sequences generated by Transformers exhibit significant statistical randomness properties.
	\item To rigorously analyze the capabilities of Transformers as pseudo-random number generators, we focus on their inherent non-linear transformations, which enhance the unpredictability of the generated sequences and potentially improve security. Consequently, we implement a Transformer model that simulates the MT algorithm for generating 16-bit pseudo-random numbers. The generated sequences successfully pass the majority of NIST statistical test suites.
	\item While Transformers can effectively learn complex statistical patterns from large-scale data, this capability has significant security implications for PRNGs, as their security fundamentally depends on the unpredictability of output sequences. Leveraging this insight, we propose utilizing Transformer models for security analysis of PRNGs through prediction attacks. Specifically, we investigate whether a Transformer model, after training on a sequence of generated pseudo-random numbers, can accurately predict subsequent values. Our experimental results demonstrate that Transformers has the potential to as a tool for evaluating PRNGs security by identifying potential statistical vulnerabilities in their output sequences.
\end{itemize}

To the best of our knowledge, this work represents the first comprehensive
study of Transformer models in the context of pseudo-random number generation,
establishing both theoretical foundations and practical viability. Our findings explore new possibilities of PRNGs based on transformer.

\section{Preliminaries}
\label{sec:preliminaries}

We introduce the key notations used throughout the paper:
$n$ denotes the number of tokens in a sample sequence,
$d$ represents the dimension of embedding vectors,
$H$ is the number of attention heads in the model,
$d_k = \frac{d}{H}$ defines the dimension per attention head, and
$L$ indicates the number of layers in the Transformer model.
boldface letters are used to denote vectors, and normal letters are used to denote scalars.
Unless otherwise specified, all vectors are represented as row vectors.

\subsection{ Decoder-only Transformers}
Transformer models, introduced in 2017~\cite{vaswani2017attention}, revolutionized natural language processing through their self-attention mechanism, which effectively captures long-range sequence dependencies.

The auto-regressive Transformer, also known as the Decoder-Only Transformer
~\cite{radford2019language},
is a sequence-to-sequence neural network model as follows
defined by the following equations:
first, the input tokens $s_i$ are embedded into a $d$-dimensional vector
$ \mathbf{x}_i^0 = \text{Embed}(s_i) + \text{pos}(i) \in \mathbb{R}^d$ for $i = 1, \ldots, n$,
where $pos(i)$ is the position embedding of the $i$-th token.
Then, the Transformer model processes
the sequence of vectors through a series of attention layers.
the $l$-th layer of the Transformer is defined as
\begin{align}
	\mathbf{x}_i^l = \mathrm{Attention}^l(\mathbf{x}_i^{l-1}) + \mathrm{FFN}^l(\mathrm{Attention}^l(\mathbf{x}_i^{l-1}))
\end{align}
\begin{align}
	\mathrm{Attention}^l(\mathbf{x}_i^{l-1}) = \mathrm{Attention}^l(\mathbf{x}_i^{l-1}) + \mathbf{x}_i^{l-1}
\end{align}
\begin{align}
	\mathrm{FFN}^l(\mathbf{x}) = \mathrm{GeLU}(\mathbf{x}\mathbf{W}_1^l + \mathbf{b}_1^l)\mathbf{W}_2^l
\end{align}
where $\mathrm{Attention}$ is the residual multi-head attention layer and $\mathrm{FFN}$ is the feed-forward network,
and $\mathbf{W}_1^l, \mathbf{W}_2^l \in \mathbb{R}^{d \times d}$, $\mathbf{b}_1^l, \mathbf{b}_2^l \in \mathbb{R}^d$
are the weights and biases of the $l$-th layer.
$\mathrm{Attention}^l(\mathbf{x}_i^{l-1})$ is given by
\begin{align*}
	\sum_{h=1}^H \mathrm{softmax}\left(\frac{(\mathbf{x}_i^{l-1}W_q^{(l,h)})(\mathbf{X}^{l-1}W_k^{(l,h)})^T}{\sqrt{d_k}}\right)(\mathbf{X}^{l-1}W_v^{(l,h)}),
\end{align*}
where $\mathbf{W}_q^{(l,h)}, \mathbf{W}_k^{(l,h)}, \mathbf{W}_v^{(l,h)} \in \mathbb{R}^{d \times d_k}$ are the weights of the query, key, and value vectors,
$\mathbf{q}_i^{(l,h)} = \mathbf{x}_i^{l-1}\mathbf{W}_q^{(l,h)}$, $\mathbf{k}_j^{(l,h)} = \mathbf{x}_j^{l-1}\mathbf{W}_k^{(l,h)}$, $\mathbf{v}_j^{(l,h)} = \mathbf{x}_j^{l-1}\mathbf{W}_v^{(l,h)}
	\in \mathbb{R}^{d_k}$
are the query, key, and value vectors,
$\mathbf{X}^{l-1} = [(\mathbf{x}_1^{l-1})^T, \ldots, (\mathbf{x}_n^{l-1})^T]^T \in \mathbb{R}^{n \times d}$ is the matrix of the key vectors.

\subsection{Pseudo-random Number Generator}

Random Number Generation can be categorized into two main types~\cite{smid2010statistical}:
True Random Number Generators (TRNGs) generate unpredictable sequences using physical entropy sources, while PRNGs are deterministic algorithms that approximate randomness based on initial seed values. For cryptographic security, PRNG seeds must come from TRNGs.
\subsubsection{The linear Congruential Generator}

The linear congruential generator is a simple and fast PRNG~\cite{maclaren1970art}.
defined as follows:
\begin{align}
	\mathbf{x}_{n+1} = (\mathbf{a}\mathbf{x}_n + \mathbf{c}) \bmod m
\end{align}
where $\mathbf{a}, \mathbf{c} \in \mathbb{Z}^d$, $m \in \mathbb{Z}$, and parameters $\mathbf{a}$ and $m$ determine the sequence period length. $\mathbf{x}0$ is the initial seed and $\mathbf{x}_i$ is the $i$-th generated number.

\subsubsection{The Mersenne Twister Algorithm}

The Mersenne twister algorithm is mentioned by Makoto Matsumoto and Takuji Nishimura in 1998~\cite{matsumoto1998mersenne}, improved upon previous PRNGs by offering high-quality random numbers with a long period ($2^{19937}-1$) and fast generation speed.

The MT algorithm is defined as two parts:
first, the initial values $\mathbf{x}_0, \ldots, \mathbf{x}_{n-1}$ are generated using LCG,
$\mathbf{x}_i \in \{0, 1\}^d$.
Then the subsequent values are generated through the following recurrence relation
including the rotation and extraction operations:

\paragraph{Rotation}
The rotation operation rotates the variables to generate the next state of
the MT algorithm.
\begin{align}
	\mathbf{t} \gets (\mathbf{x}[i] \, \land \, \mathbf{upper}) \, \lor \, (\mathbf{x}[(i+1)\, \bmod \, n] \, \land \, \mathbf{lower})
\end{align}
\begin{align}
	\mathbf{z} \gets \mathbf{x}[(i+m)\, \bmod \, n] \, \oplus \, (\mathbf{t} >> 1) \, \oplus \, \begin{cases}
		                                                                                   \mathbf{0} & \text{if } t_0 = 0 \\
		                                                                                   \mathbf{a} & \text{otherwise}
	                                                                                   \end{cases}
\end{align}
where $\mathbf{upper}= \mathbf{1}<< (d-r)$ and $\mathbf{lower}= \neg \mathbf{upper}$,$\mathbf{a}$ are constants vectors.

\paragraph{Extraction}
The extraction operation generates the next pseudorandom number through a series of bit-wise operations.
\begin{align}
	\mathbf{x}[i] & \gets \mathbf{z}                                                  \\
	\mathbf{y}    & \gets \mathbf{x}[i]                                               \\
	\mathbf{y}    & \gets \mathbf{y} \, \oplus \, (\mathbf{y} >> u)                      \\
	\mathbf{y}    & \gets \mathbf{y} \, \oplus \, ((\mathbf{y} << s)\, \land \, \mathbf{b}) \\
	\mathbf{y}    & \gets \mathbf{y} \, \oplus \, ((\mathbf{y} << t)\, \land \, \mathbf{c}) \\
	\mathbf{y}    & \gets \mathbf{y} \, \oplus \, (\mathbf{y} >> l)
\end{align}
where $u,l,s,t$ are constants, $\mathbf{b},\mathbf{c}$ are constant vectors.
The above bit-wise operations generate the next pseudorandom number, operating on each bit individually.

\subsection{Circuit Complexity}
We begin by introducing several fundamental circuit complexity classes, arranged in order of increasing computational power. For a comprehensive treatment, we refer readers to ~\cite{arora2009computational}.

\begin{itemize}
	\item $\text{TC}^{k-1}$ comprises languages computed by circuits with $O(\log^{k-1} n)$ depth, polynomial size, and unbounded fan-in MAJORITY gates (with NOTs available at no cost). A MAJORITY gate outputs the most frequent input value, defaulting to 1 in case of ties. Of particular interest is the class $\text{TC}^0$.

	\item $\text{AC}^{k-1}[m]$ consists of languages computed by constant-depth, polynomial-size circuits with unbounded fan-in, using AND, OR, and $\text{MOD}m$ gates. A $\text{MOD}m$ gate outputs 1 if and only if the sum of its input bits is divisible by $m$. While NOT operations can be simulated using $\text{MOD}m$ gates, it remains an open question whether AND (or OR) operations can be simulated using only $\text{MOD}m$ gates in constant depth.

	\item $\text{ACC} = \bigcup_m \text{AC}^0[m]$, representing the union of all $\text{AC}^0[m]$ classes.

	\item $\text{AC}^{k-1}$ is defined as $\text{AC}^{k-1}[m]$ without $\text{MOD}~m$ gates.

\end{itemize}

These circuit complexity classes form a strict hierarchy, characterized by the following containment relations:
\[ \text{AC}^0 \subseteq \text{AC}^0[m] \subseteq \text{ACC} \subseteq \text{TC}^0 \]

Circuit classes can be non-uniform, using different algorithms for different input lengths, or uniform, using a single algorithm for all inputs. In the uniform model, an efficient algorithm can generate the appropriate n-th circuit for any n-bit input, independent of input length. While uniform versions exist for all circuit complexity classes, this paper focuses on the non-uniform class $\text{AC}^0$.

\section{PRNGs via Transformer}
\label{sec:expressiveness}

We adopt a widely used and realistic setting known as the log-precision Transformer~\cite{merrill2023parallelism,liu2022transformers}, in which all parameters of the Transformer are constrained to $O(\log(n))$ bit precision, where $n$ represents the maximum length of the input sequence (see \Cref{sec:log_precision} for further details).

The proofs of both theorems are constructive, providing explicit
implementations of PRNGs using Transformer architectures. Our approach
centers on utilizing the fundamental components of the Transformer
architecture—specifically the attention mechanism and feed-forward
network (FFN) layers—to systematically realize essential computational
operations. These operations encompass both logical operations
(XOR, NOT, AND, OR) and arithmetic operations (multiplication, addition, modulus)
performed on pseudo-random numbers and constants. Our construction
methodology extends the theoretical foundations established by
Feng et al.~\cite{feng2024towards} and Yang et al.~\cite{yang2024efficient},
particularly their seminal work on the function approximation capabilities
of FFN. The complete proofs are presented in \Cref{sec:technical_lemmas}.

\subsection{Main Theoretical Results}

We first present the implementation of fundamental Boolean operations using Transformer
architecture, where \( \phi(i) \) denotes the binary representation of the i-th number.

\begin{lemma}
	\label{lemma:variable_and_variable}
	For a given variable vector \( ( \phi(i), \phi(j), i, 1 ) \),
	there exists a Transformer layer that approximates
	function \( f(i, j) \) with a constant \( \epsilon > 0 \) such that
	\( \|f(i, j) - \phi(i) \land \phi(j)\|_\infty \leq \epsilon \).
	The parameters of this layer exhibit an \( \ell_\infty \) norm bounded by
	\( O(poly(M, 1/\epsilon)) \), where M denotes the upper bound for all parameter values.
\end{lemma}

\begin{lemma}
	\label{lemma:variable_or_variable}
	For a given variable vector \( ( \phi(i), \phi(j), i, 1 ) \),
	there exists a Transformer layer that approximates
	function \( f(i, j) \) with a constant \( \epsilon > 0 \) such that
	\( \|f(i, j) - \phi(i) \lor \phi(j)\|_\infty \leq \epsilon \).
	The parameters of this layer exhibit an \( \ell_\infty \) norm bounded by
	\( O(poly(M, 1/\epsilon)) \), where M denotes the upper bound for all parameter values.
\end{lemma}

\begin{lemma}
	\label{lemma:not_variable}
	Given a variable vector \( ( \phi(i), i, 1) \),
	there exists a attention layer that can approximate
	the boolean function \( \neg \phi(i) \).
\end{lemma}

\begin{lemma}
	\label{lemma:variable_xor_variable}
	For a given variable vector \( ( \phi(i), \phi(j), i, 1 ) \),
	there exists a Transformer layer that approximates
	function \( f(i, j) \) with a constant \( \epsilon > 0 \) such that
	\( \|f(i, j) - \phi(i) \oplus \phi(j)\|_\infty \leq \epsilon \).
	The parameters of this layer exhibit an \( \ell_\infty \) norm bounded by
	\( O(poly(M, 1/\epsilon)) \), where M denotes the upper bound for all parameter values.
\end{lemma}

Building on the aforementioned lemmas, we can utilize the Transformer module to implement the specific steps of both the LCG and MT algorithms, thereby completing our proof.

\begin{theorem}\label{thm:LCG}
	For the linear congruential generator algorithm,
	we can construct an autoregressive Transformer as defined
	in ~\Cref{sec:preliminaries} that is capable of
	simulating and generating $n$ pseudo-random numbers.
	The constructed Transformer has a hidden dimension of $d = O(n)$,
	consists of \textbf{one layer} with \textbf{one attention head},
	and all parameters are polynomially bounded by $O(poly(n))$.
\end{theorem}

\begin{proof}[Proof Sketch]
	We utilize the attention module of the Transformer to implement the multiplication of the pseudo-random number \( x_{i-1} \) by the constant \( a \) (as established in the lemma). This result is then added to the constant \( c \), followed by the application of the FFN module to perform the modulus operation \( \mod(m) \). The detailed proof can be found in \cref{sec:proof_expressiveness}.
\end{proof}

In \Cref{thm:MT}, we employed chain of thought to enable the Transformer to simulate the intermediate steps of the MT algorithm sequentially.

\begin{theorem}\label{thm:MT}
	For the Mersenne Twister algorithm,
	we can construct an autoregressive Transformer with Chain of Thought as defined
	in ~\Cref{sec:preliminaries} that is capable of
	simulating and generating $n$ pseudo-random numbers.
	The constructed Transformer has a hidden dimension of $d = O(n)$,
	consists of \textbf{seventeen layers}
	with at most \textbf{four attention heads per layer},
	and all parameters are polynomially bounded by $O(poly(n))$.
\end{theorem}

\begin{proof}[Proof Sketch]
	Initially, we employ an attention layer to compute the total count of
	generated pseudo-random numbers, denoted as $cnt_{\Rightarrow}$.
	Subsequently, utilizing this $cnt_{\Rightarrow}$ in conjunction with
	the attention layer, we determine the requisite variable indices for
	generating new pseudo-random numbers in the MT algorithm.
	Based on these indices, we retrieve the necessary random numbers and
	implement the rotation and extraction operations characteristic of the
	MT algorithm using a multi-layer Transformer architecture.
	Finally, we implement an MLP to distinguish between symbolic output and
	pseudo-random number generation. The complete proof is presented in \Cref{sec:proof_expressiveness}.
\end{proof}

We can also simulate basic Boolean operations using Transformers, demonstrating that log-precision autoregressive Transformers can express non-uniform $\text{AC}^0$.

\begin{corollary}
	\label{cor:transformer_ac0}
	Log-precision autoregressive Transformers with polynomial size can simulate
	any circuit family in non-uniform $\text{AC}^0$. Specifically, for any
	circuit family $\{C_n\}_{n \in \mathbb{N}}$ in $\text{AC}^0$ with size
	$\text{poly}(n)$, there exists a family of log-precision decoder-only
	Transformers $\{T_n\}_{n \in \mathbb{N}}$ of size $\text{poly}(n)$ that
	can simulate $\{C_n\}_{n \in \mathbb{N}}$.
\end{corollary}

\begin{proof}
	Leveraging \Cref{lemma:variable_and_variable} and \Cref{lemma:variable_or_variable},
	we demonstrate that a polynomial-sized Transformer layer can simulate the
	fundamental circuit gates AND and OR in $\text{AC}^0$. This result establishes
	that log-precision autoregressive Transformers possess at least the expressive
	power of non-uniform $\text{AC}^0$.
\end{proof}

In the aforementioned theorems, we provided detailed constructions using
Transformer architecture to implement each step of both the LCG and MT
pseudo-random number generation algorithms, thereby establishing a theoretical
foundation for Transformers' capability to simulate some pseudo-random generators. Subsequently, in ~\Cref{sec:experiments}, we validate our theoretical findings through comprehensive empirical evaluation, designing experiments to assess both the effectiveness of Transformers as pseudo-random number generators and their potential in prediction-based attack scenarios.

\section{Experiments}
\label{sec:experiments}

Although we analyze the MT algorithm and the LCG can be simulated by transformer
with constant depth in previous section,
it is still important to verify the effectiveness of our theoretical analysis
by experiments. In this section, we will design several experiments.
First, we will verify the effectiveness of our theoretical analysis on the MT algorithm.
Then, we find that generator based on transformer can pass the most statistical tests designed by NIST.
Finally, we explore how transformers can be used for prediction attacks, such as predicting the output of the MT algorithm. All experiments related to training and sequence generation were conducted on a machine running the RokeyLinux operating system, equipped with an A800 GPU with 80 GB of memory. The NIST statistical tests were performed separately on a personal laptop computer.

\subsection{Transformer Simulation of Mersenne Twister}
\paragraph{Model and Dataset}
We use the MT 19937 algorithm to generate
32bit pseudo-random numbers(mt19937-32) as our base dataset.
These numbers were subsequently transformed into 8-bit, 12-bit, and 16-bit representations
through modular arithmetic operations. The detailed dataset configurations are
presented in~\Cref{tab:mt19937-dataset}.

\begin{table}[h]
	\centering
	\begin{tabular}{ccc} %
		\toprule
		Dataset       & Sequence Length & Training Set Size \\
		\midrule
		mt19937-8bit  & 256             & 256               \\
		mt19937-12bit & 4096            & 4096              \\
		mt19937-16bit & 4096            & 4096              \\
		\bottomrule
	\end{tabular}
	\caption{Dataset configuration}
	\label{tab:mt19937-dataset}
\end{table}

To evaluate the simulation capability, we trained a GPT-2 architecture
on the aforementioned datasets. The model was trained to minimize the cross-entropy
loss, with training accuracy serving as our primary metric for assessing the quality
of PRNG simulation. We conducted extensive experiments with various hyperparameter
configurations, and the detailed model architecture specifications are summarized
in~\Cref{tab:mt19937-model}.

\begin{table}[h]
	\centering
	\begin{tabular}{ccc}
		\toprule
		Parameter    & Value & Description                  \\
		\midrule
		attn\_pdrop  & 0.1   & Dropout for attention layers \\
		embd\_pdrop  & 0.1   & Dropout for embeddings       \\
		model\_type  & gpt2  & Architecture type            \\
		n\_embd      & 768   & Embedding dimension          \\
		n\_head      & 12    & Attention heads              \\
		n\_layer     & 12    & Transformer layers           \\
		n\_positions & 1024  & Max sequence length          \\
		resid\_pdrop & 0.1   & Dropout for residuals        \\
		vocab\_size  & 50257 & Vocabulary size              \\
		\bottomrule
	\end{tabular}
	\caption{Model configuration}
	\label{tab:mt19937-model}
\end{table}

\begin{figure}[!ht]
    \centering
	\begin{minipage}{0.8\textwidth}
		\subfigure[Training Accuracy]{
			\includegraphics[width=0.7\textwidth]{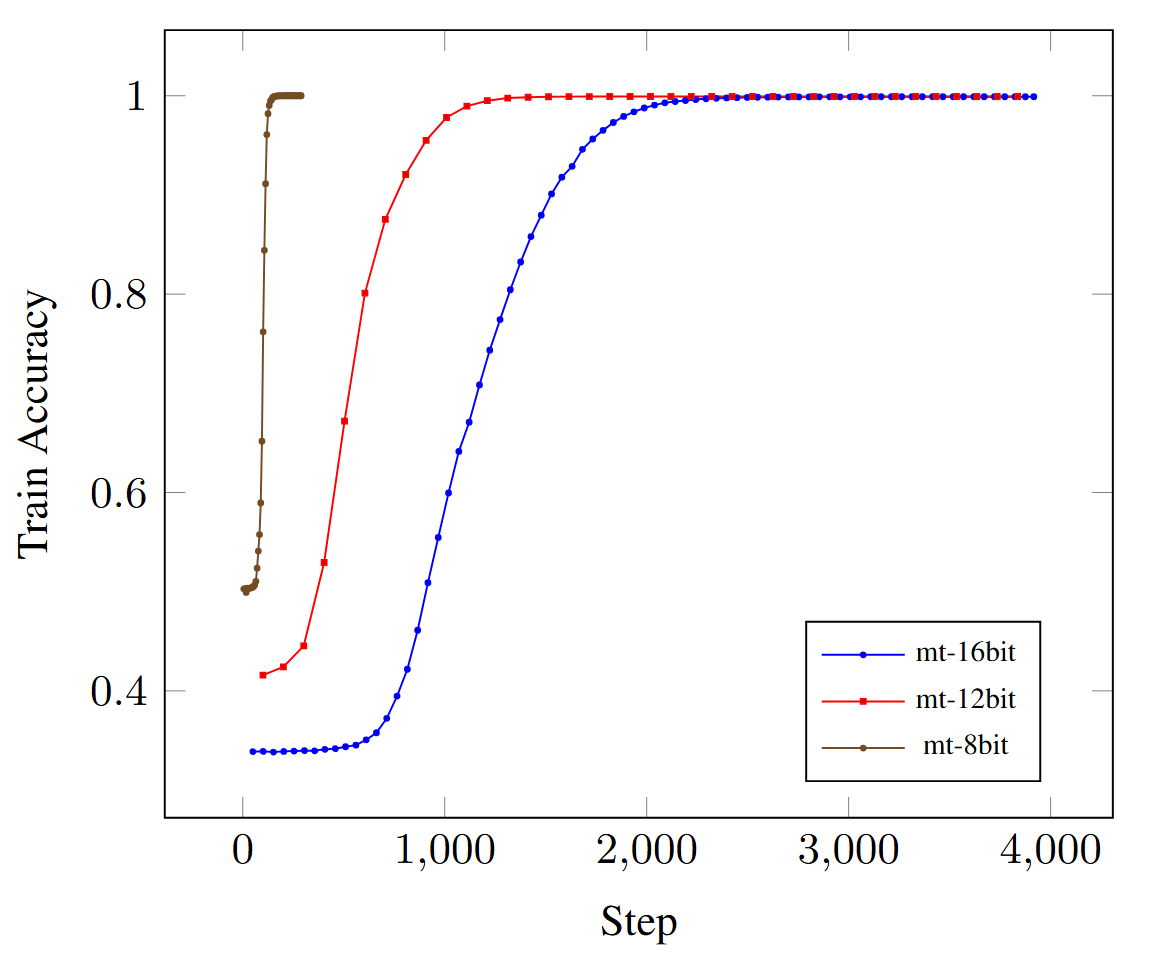}
		}
	\end{minipage}
	\hfill
	\begin{minipage}{0.8\textwidth}
		\subfigure[Training Loss]{
			\includegraphics[width=0.7\textwidth]{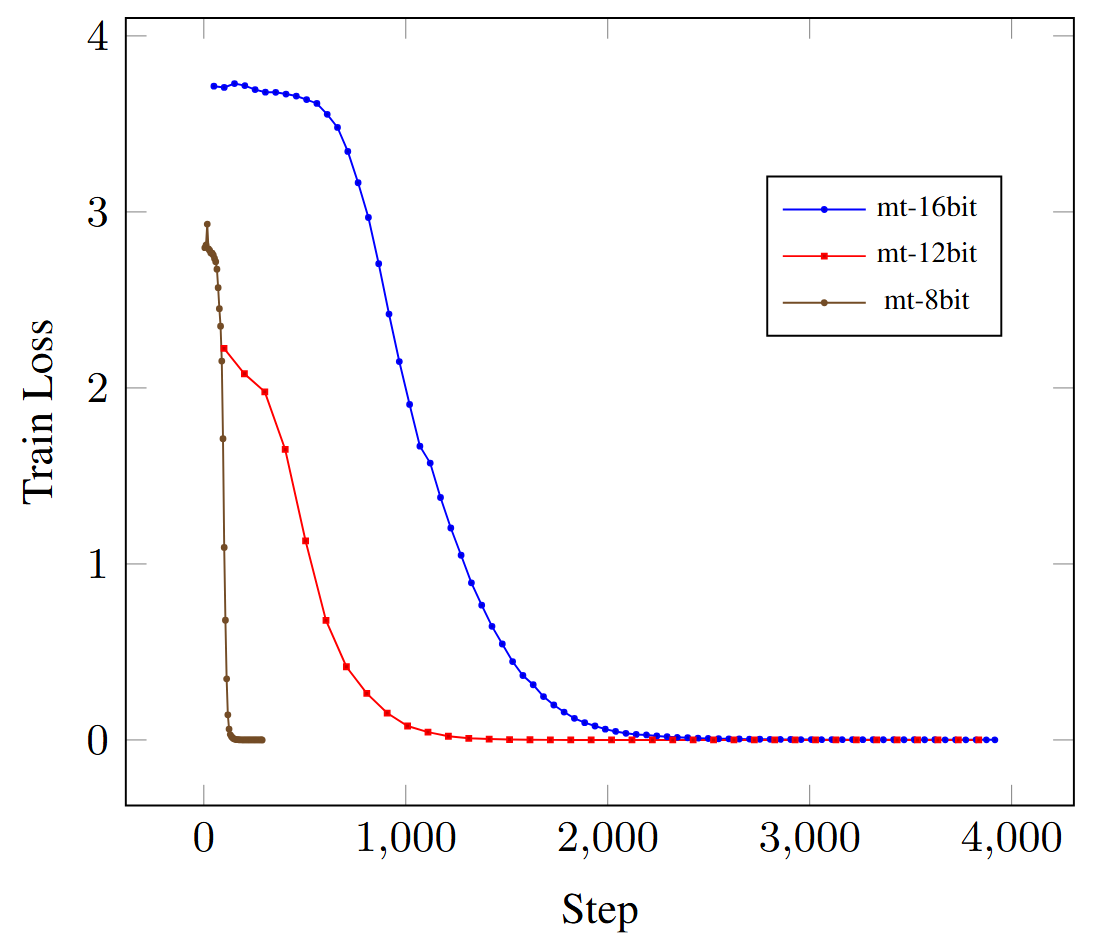}
			}
	\end{minipage}
	\caption{Training Accuracy and Loss}
	\label{fig:mt19937-train}
\end{figure}


\paragraph{Results}

The experimental results demonstrate consistent convergence patterns across different bit-width configurations (8-bit, 12-bit, and 16-bit) of the MT simulation. As shown in ~\Cref{fig:mt19937-train}, the training accuracy curves exhibit sigmoid-like behavior, with the 8-bit model converging most rapidly (within 500 steps), followed by the 12-bit model (approximately 1,000 steps), and the 16-bit model requiring the longest training period (around 2,000 steps) to achieve convergence. Correspondingly, the training loss curves display exponential decay, with all three models eventually stabilizing at minimal loss values, indicating successful model optimization. Notably, all configurations ultimately achieve near-perfect accuracy (approximately 1.0), suggesting that the Transformer architecture can effectively simulate the MT algorithm regardless of bit-width, albeit with varying convergence rates inversely proportional to the complexity of the target bit-width.

\begin{figure}[!ht]
	\centering
	\subfigure[Before training]{
		\includegraphics[width=0.7\textwidth]{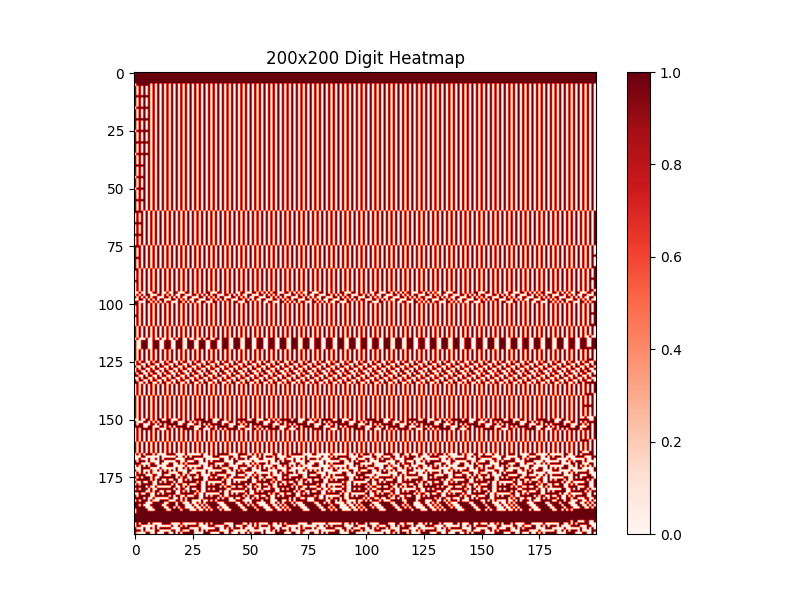}
		\label{fig:mt19937-heatmap-before}
	}
	\subfigure[After training]{
		\includegraphics[width=0.7\textwidth]{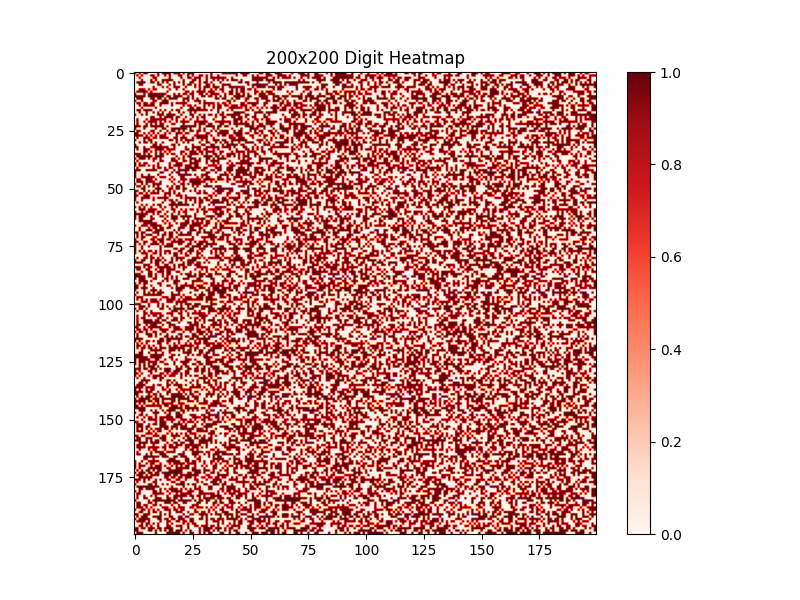}
		\label{fig:mt19937-heatmap-after}
	}
	\caption{heatmap of bit stream}
	\label{fig:mt19937-heatmap}
\end{figure}

\subsection{NIST Statistical Tests for Randomness}

\paragraph{NIST Statistical Test Suite}

We evaluated our pseudo-random sequences using the NIST Statistical Test Suite (NIST STS) ~\cite{smid2010statistical}, which includes 15 tests: Frequency, Block Frequency, Cumulative Sums, Runs, Longest Run of Ones, Rank, Discrete Fourier Transform, Nonperiodic Template Matchings, Overlapping Template Matchings, Universal Statistical, Approximate Entropy, Random Excursions, Random Excursions Variant, Serial, and Linear Complexity. Each test evaluates specific statistical properties using rigorous mathematical principles to detect potential non-random patterns and statistical deviations in binary sequences.

Our experimental methodology involved training a Transformer model on a 16-bit dataset, with model checkpoints systematically preserved at training accuracies thresholds of 0.63, 0.68, 0.81, and 0.89. The pseudo-random sequences generated from each checkpoint were subsequently subjected to the complete NIST STS battery of tests.

Our NIST testing passed 11 of 15 tests (corresponding to training accuracies of 0.63, 0.68, 0.81, and 0.89, respectively). The results of the NIST Statistical Test Suite that passed, as illustrated in~\Cref{fig:pvalues}, indicate that the generated pseudo-random sequences exhibit satisfactory statistical properties. Specifically, all tests, except for the Block Frequency test, yielded p-values exceeding the significance level of 0.01, suggesting that the sequences successfully passed the majority of the statistical tests. This outcome implies that the Transformer model effectively generates sequences that maintain randomness characteristics, thereby validating the robustness of the simulation process employed in this study.

The test failures reveal several limitations in our generator: RandomExcursions and RandomExcursionsVariant tests show insufficient random walk characteristics and missing P-values due to inadequate sequence variations; the Universal Test fails due to sequence compressibility issues from pattern reproduction in training data; and the Approximate Entropy Test fails due to limited sequence complexity caused by fixed-length training sequences.Notably, these results just achieved using merely a basic GPT-2 architecture, highlighting the considerable potential of Transformer-based approaches in pseudo-random number generation.

The heatmap presented in ~\Cref{fig:mt19937-heatmap}
illustrates the transformation of the generated pseudo-random sequences before and after the training process.
~\Cref{fig:mt19937-heatmap-before} reveals a distinct pattern characterized by regularity and structure, indicative of a non-random distribution of values.
Conversely, ~\Cref{fig:mt19937-heatmap-after} demonstrates a significant shift towards randomness, as evidenced by the chaotic and uniform distribution of values across the grid. This stark contrast underscores the effectiveness of the training regimen, which successfully enhances the stochastic properties of the output sequences, thereby validating the capability of the Transformer model to approximate true randomness in the generated bit stream.

\begin{figure*}[!htbp]
	\centering
	\includegraphics[width=1.0\textwidth]{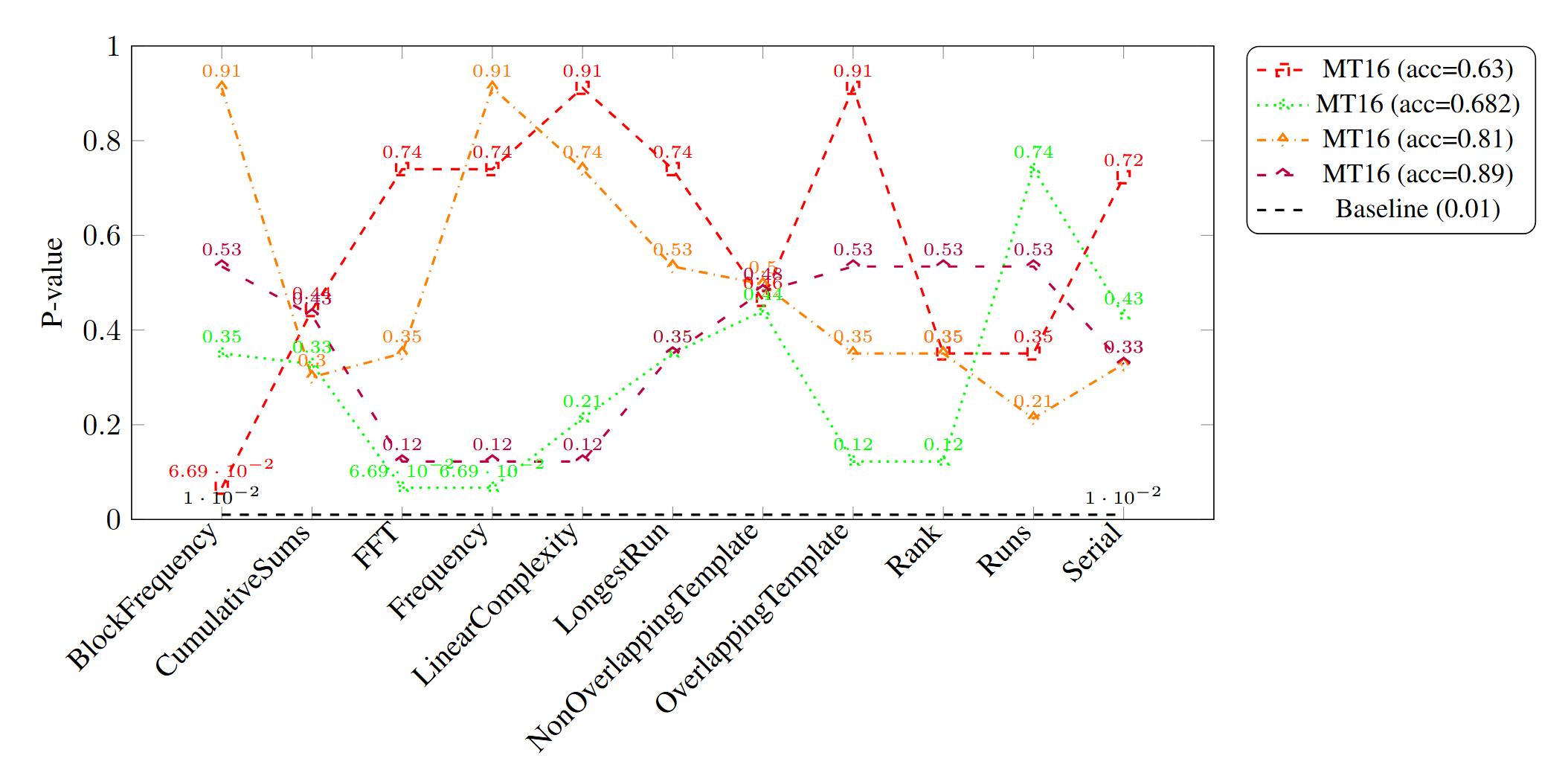}
	\caption{P-values for Different Statistical Tests}
	\label{fig:pvalues}
\end{figure*}

\subsection{Prediction Attack on Mersenne Twister}
Building upon ~\cite{tao2025can}'s findings on Transformers' ability to utilize prime factorizations and RNS representations in LCG sequence prediction,
we employed a Transformer-based architecture to model and predict the output sequences of the Mersenne Twister (MT) algorithm. The experimental setup involved training the model with both 8-bit and 12-bit random number sequences, utilizing training-to-test set ratios of 1:10 and 1:20, respectively. Our empirical results demonstrate that the Transformer model achieves prediction accuracy ranging from 0.7 to 0.8 when inferring MT algorithm outputs. The detailed performance metrics and convergence characteristics are illustrated in~\Cref{fig:mt19937-prediction}.
These results not only advance our understanding of Transformer capabilities in sequence prediction but also provide valuable insights for developing more secure cryptographic algorithms.

\begin{figure}[!ht]
	\centering
	\includegraphics[width=0.7\textwidth]{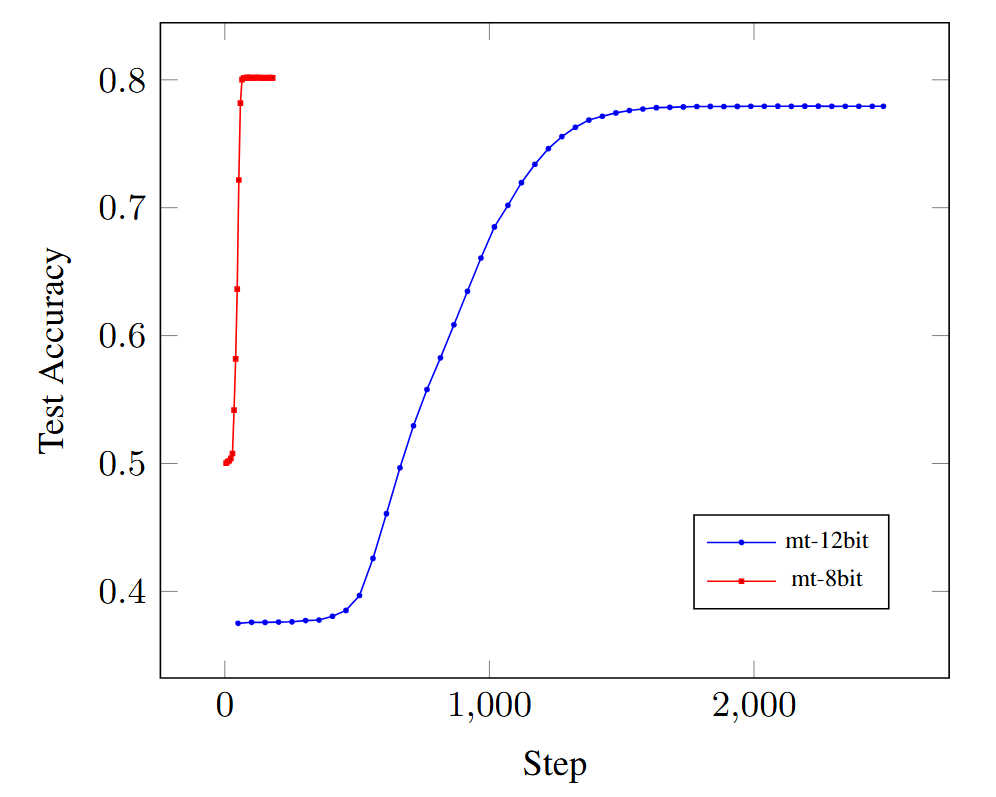}
	\caption{Prediction Accuracy}
	\label{fig:mt19937-prediction}
\end{figure}

\section{Related Work}
\label{sec:related}

Our work investigates the expressive power and limitations of Transformers,
with a specific focus on pseudo-random number generators. Transformer expressiveness research evolved in two main phases. Initially, ~\cite{yun2019transformers} proved Transformers' universal approximation capabilities for continuous sequence functions, later extended to Sparse~\cite{yun2020n} and Linear Transformers~\cite{alberti2023sumformer}. Recent work focuses on in-context learning, with ~\cite{garg2022can,dai2022can,von2023transformers} demonstrating basic function learning, while ~\cite{feng2024towards} showed Chain of Thought~\cite{wei2022chain} capabilities for linear equations and dynamic programming. Further studies revealed their ability to solve P-class~\cite{merrill2023expresssive} and P/poly problems~\cite{li2024chain}. For a complete review, see ~\cite{yang2024efficient}.

From the perspective of constructing PRNGs, researchers have extensively explored various machine learning architectures to develop novel generators. These approaches can be categorized into three main frameworks:
(1) Recurrent architectures have been widely adopted, leveraging their sequential processing capabilities. These include LSTM-based approaches that excel at capturing long-term dependencies~\cite{jeong2018pseudo,pasqualini2020pseudo,jeong2020pseudo}, traditional RNNs for basic sequence generation~\cite{desai2012pseudo,desai2012using}, Elman networks for their feedback mechanisms~\cite{desai2011pseudo}, and Hopfield networks known for their stability properties~\cite{hameed2018utilizing,lin2023triple,bao2024two}.
(2) Generative approaches using GANs have demonstrated promising results by employing adversarial training to improve random number quality~\cite{de2019pseudo,oak2019poster,wu2025gan,okada2023learned}. The generator-discriminator architecture helps ensure the generated numbers exhibit desired statistical properties.
(3) Reinforcement learning-based methods have emerged as an innovative approach, using reward mechanisms to optimize random number generation~\cite{pasqualini2020pseudo_rl,pasqualini2020pseudo,park2022dynamical,almardeny2022reinforcement}. These methods adapt their generation strategies based on feedback about the quality of produced sequences.
For a comprehensive review of these approaches and their comparative advantages, see ~\cite{wu2025pseudorandom}.


Note that we use transformer model to simulate specific step of LCG and MT PRNGs.
Our approach differs from previous work often to simulate the abstract model.
such as simulating the Turing machine~\cite{merrill2023expresssive} to solve
P class problems and simulating the circuits~\cite{li2024chain} to solve
P/poly class problems.
It is similar to the proof technique of~\cite{feng2024towards}.

\section{Conclusion and Future Directions}
\label{sec:conclusion}

Our theoretical analysis and experimental results indicate that the non-linear
transformations and robust sequence modeling capabilities inherent in Transformer
architectures can be effectively leveraged for both PRNGs construction and
security analysis. Specifically, our experiments reveal that Transformer-based
generators can produce high-quality pseudo-random sequences that satisfy the
majority of NIST statistical criteria. Furthermore, the generated sequences
demonstrate resistance to prediction attacks, suggesting that Transformers
have potential applications both as pseudo-random number generators and as
tools for evaluating PRNGs security.

Future research on Transformer-based PRNGs will explore three key directions, building upon neural networks' established strengths in parallel processing and versatile applications: (1) Enhancing Transformers' capabilities to generate more secure and efficient pseudo-random numbers by leveraging their parallel architecture, (2) Developing Transformer-based frameworks for comprehensive PRNG security testing and evaluation, and (3) Theoretically analyzing Transformers' boundaries in fitting highly nonlinear functions, particularly identifying which types of functions remain beyond their computational reach.

\bibliography{refs}

\begin{thebibliography}{39}
\providecommand{\natexlab}[1]{#1}
\providecommand{\url}[1]{\texttt{#1}}
\expandafter\ifx\csname urlstyle\endcsname\relax
  \providecommand{\doi}[1]{doi: #1}\else
  \providecommand{\doi}{doi: \begingroup \urlstyle{rm}\Url}\fi

\bibitem[Alberti et~al.(2023)Alberti, Dern, Thesing, and Kutyniok]{alberti2023sumformer}
Silas Alberti, Niclas Dern, Laura Thesing, and Gitta Kutyniok.
\newblock Sumformer: Universal approximation for efficient transformers.
\newblock In \emph{Topological, Algebraic and Geometric Learning Workshops 2023}, pages 72--86. PMLR, 2023.

\bibitem[Almardeny et~al.(2022)Almardeny, Benavoli, Boujnah, and Naredo]{almardeny2022reinforcement}
Yahya Almardeny, Alessio Benavoli, Noureddine Boujnah, and Enrique Naredo.
\newblock A reinforcement learning system for generating instantaneous quality random sequences.
\newblock \emph{IEEE Transactions on Artificial Intelligence}, 4\penalty0 (3):\penalty0 402--415, 2022.

\bibitem[Arora and Barak(2009)]{arora2009computational}
Sanjeev Arora and Boaz Barak.
\newblock \emph{Computational complexity: a modern approach}.
\newblock Cambridge University Press, 2009.

\bibitem[Bao et~al.(2024)Bao, Tang, Su, Bao, Chen, and Xu]{bao2024two}
Bocheng Bao, Haigang Tang, Yuanhui Su, Han Bao, Mo~Chen, and Quan Xu.
\newblock Two-dimensional discrete bi-neuron hopfield neural network with polyhedral hyperchaos.
\newblock \emph{IEEE Transactions on Circuits and Systems I: Regular Papers}, 2024.

\bibitem[Dai et~al.(2022)Dai, Sun, Dong, Hao, Ma, Sui, and Wei]{dai2022can}
Damai Dai, Yutao Sun, Li~Dong, Yaru Hao, Shuming Ma, Zhifang Sui, and Furu Wei.
\newblock Why can gpt learn in-context? language models implicitly perform gradient descent as meta-optimizers.
\newblock \emph{arXiv preprint arXiv:2212.10559}, 2022.

\bibitem[De~Bernardi et~al.(2019)De~Bernardi, Khouzani, and Malacaria]{de2019pseudo}
Marcello De~Bernardi, MHR Khouzani, and Pasquale Malacaria.
\newblock Pseudo-random number generation using generative adversarial networks.
\newblock In \emph{ECML PKDD 2018 Workshops: Nemesis 2018, UrbReas 2018, SoGood 2018, IWAISe 2018, and Green Data Mining 2018, Dublin, Ireland, September 10-14, 2018, Proceedings 18}, pages 191--200. Springer, 2019.

\bibitem[Desai et~al.(2012{\natexlab{a}})Desai, Patil, Deshmukh, and Rao]{desai2012pseudo}
V~Desai, Ravindra~T Patil, V~Deshmukh, and D~Rao.
\newblock Pseudo random number generator using time delay neural network.
\newblock \emph{World}, 2\penalty0 (10):\penalty0 165--169, 2012{\natexlab{a}}.

\bibitem[Desai et~al.(2012{\natexlab{b}})Desai, Patil, and Rao]{desai2012using}
Veena Desai, Ravindra Patil, and Dandina Rao.
\newblock Using layer recurrent neural network to generate pseudo random number sequences.
\newblock \emph{International Journal of Computer Science Issues}, 9\penalty0 (2):\penalty0 324--334, 2012{\natexlab{b}}.

\bibitem[Desai et~al.(2011)Desai, Deshmukh, and Rao]{desai2011pseudo}
VV~Desai, VB~Deshmukh, and DH~Rao.
\newblock Pseudo random number generator using elman neural network.
\newblock In \emph{2011 IEEE Recent Advances in Intelligent Computational Systems}, pages 251--254. IEEE, 2011.

\bibitem[Feng et~al.(2024)Feng, Zhang, Gu, Ye, He, and Wang]{feng2024towards}
Guhao Feng, Bohang Zhang, Yuntian Gu, Haotian Ye, Di~He, and Liwei Wang.
\newblock Towards revealing the mystery behind chain of thought: a theoretical perspective.
\newblock \emph{Advances in Neural Information Processing Systems}, 36, 2024.

\bibitem[Garg et~al.(2022)Garg, Tsipras, Liang, and Valiant]{garg2022can}
Shivam Garg, Dimitris Tsipras, Percy~S Liang, and Gregory Valiant.
\newblock What can transformers learn in-context? a case study of simple function classes.
\newblock \emph{Advances in Neural Information Processing Systems}, 35:\penalty0 30583--30598, 2022.

\bibitem[Gentle(2003)]{gentle2003random}
James~E Gentle.
\newblock \emph{Random number generation and Monte Carlo methods}, volume 381.
\newblock Springer, 2003.

\bibitem[Hameed and Ali(2018)]{hameed2018utilizing}
Sarab~M Hameed and Layla M~Mohammed Ali.
\newblock Utilizing hopfield neural network for pseudo-random number generator.
\newblock In \emph{2018 IEEE/ACS 15th International Conference on Computer Systems and Applications (AICCSA)}, pages 1--5. IEEE, 2018.

\bibitem[Jeong et~al.(2018)Jeong, Oh, Cho, and Choi]{jeong2018pseudo}
Young-Seob Jeong, Kyojoong Oh, Chung-Ki Cho, and Ho-Jin Choi.
\newblock Pseudo random number generation using lstms and irrational numbers.
\newblock In \emph{2018 IEEE international conference on big data and smart computing (BigComp)}, pages 541--544. IEEE, 2018.

\bibitem[Jeong et~al.(2020)Jeong, Oh, Cho, and Choi]{jeong2020pseudo}
Young-Seob Jeong, Kyo-Joong Oh, Chung-Ki Cho, and Ho-Jin Choi.
\newblock Pseudo-random number generation using lstms.
\newblock \emph{The Journal of Supercomputing}, 76:\penalty0 8324--8342, 2020.

\bibitem[Li et~al.(2024)Li, Liu, Zhou, and Ma]{li2024chain}
Zhiyuan Li, Hong Liu, Denny Zhou, and Tengyu Ma.
\newblock Chain of thought empowers transformers to solve inherently serial problems.
\newblock \emph{arXiv preprint arXiv:2402.12875}, 2024.

\bibitem[Lin et~al.(2023)Lin, Wang, Yu, Hong, Xu, and Sun]{lin2023triple}
Hairong Lin, Chunhua Wang, Fei Yu, Qinghui Hong, Cong Xu, and Yichuang Sun.
\newblock A triple-memristor hopfield neural network with space multistructure attractors and space initial-offset behaviors.
\newblock \emph{IEEE Transactions on Computer-Aided Design of Integrated Circuits and Systems}, 42\penalty0 (12):\penalty0 4948--4958, 2023.

\bibitem[Liu et~al.(2022)Liu, Ash, Goel, Krishnamurthy, and Zhang]{liu2022transformers}
Bingbin Liu, Jordan~T Ash, Surbhi Goel, Akshay Krishnamurthy, and Cyril Zhang.
\newblock Transformers learn shortcuts to automata.
\newblock \emph{arXiv preprint arXiv:2210.10749}, 2022.

\bibitem[MacLaren(1970)]{maclaren1970art}
M~Donald MacLaren.
\newblock The art of computer programming. volume 2: Seminumerical algorithms (donald e. knuth).
\newblock \emph{SIAM Review}, 12\penalty0 (2):\penalty0 306--308, 1970.

\bibitem[Matsumoto and Nishimura(1998)]{matsumoto1998mersenne}
Makoto Matsumoto and Takuji Nishimura.
\newblock Mersenne twister: a 623-dimensionally equidistributed uniform pseudo-random number generator.
\newblock \emph{ACM Transactions on Modeling and Computer Simulation (TOMACS)}, 8\penalty0 (1):\penalty0 3--30, 1998.

\bibitem[Merrill and Sabharwal(2023{\natexlab{a}})]{merrill2023expresssive}
William Merrill and Ashish Sabharwal.
\newblock The expresssive power of transformers with chain of thought.
\newblock \emph{arXiv preprint arXiv:2310.07923}, 2023{\natexlab{a}}.

\bibitem[Merrill and Sabharwal(2023{\natexlab{b}})]{merrill2023parallelism}
William Merrill and Ashish Sabharwal.
\newblock The parallelism tradeoff: Limitations of log-precision transformers.
\newblock \emph{Transactions of the Association for Computational Linguistics}, 11:\penalty0 531--545, 2023{\natexlab{b}}.

\bibitem[Oak et~al.(2019)Oak, Rahalkar, and Gujar]{oak2019poster}
Rajvardhan Oak, Chaitanya Rahalkar, and Dhaval Gujar.
\newblock Poster: Using generative adversarial networks for secure pseudorandom number generation.
\newblock In \emph{Proceedings of the 2019 ACM SIGSAC Conference on Computer and Communications Security}, pages 2597--2599, 2019.

\bibitem[Okada et~al.(2023)Okada, Endo, Yasuoka, and Kurabayashi]{okada2023learned}
Kiyoshiro Okada, Katsuhiro Endo, Kenji Yasuoka, and Shuichi Kurabayashi.
\newblock Learned pseudo-random number generator: Wgan-gp for generating statistically robust random numbers.
\newblock \emph{PloS one}, 18\penalty0 (6):\penalty0 e0287025, 2023.

\bibitem[Park et~al.(2022)Park, Kim, Kim, and Nam]{park2022dynamical}
Sungju Park, Kyungmin Kim, Keunjin Kim, and Choonsung Nam.
\newblock Dynamical pseudo-random number generator using reinforcement learning.
\newblock \emph{Applied Sciences}, 12\penalty0 (7):\penalty0 3377, 2022.

\bibitem[Pasqualini and Parton(2020{\natexlab{a}})]{pasqualini2020pseudo}
Luca Pasqualini and Maurizio Parton.
\newblock Pseudo random number generation through reinforcement learning and recurrent neural networks.
\newblock \emph{Algorithms}, 13\penalty0 (11):\penalty0 307, 2020{\natexlab{a}}.

\bibitem[Pasqualini and Parton(2020{\natexlab{b}})]{pasqualini2020pseudo_rl}
Luca Pasqualini and Maurizio Parton.
\newblock Pseudo random number generation: A reinforcement learning approach.
\newblock \emph{Procedia Computer Science}, 170:\penalty0 1122--1127, 2020{\natexlab{b}}.

\bibitem[Press(2007)]{press2007numerical}
William~H Press.
\newblock \emph{Numerical recipes 3rd edition: The art of scientific computing}.
\newblock Cambridge university press, 2007.

\bibitem[Radford et~al.(2019)Radford, Wu, Child, Luan, Amodei, Sutskever, et~al.]{radford2019language}
Alec Radford, Jeffrey Wu, Rewon Child, David Luan, Dario Amodei, Ilya Sutskever, et~al.
\newblock Language models are unsupervised multitask learners.
\newblock \emph{OpenAI blog}, 1\penalty0 (8):\penalty0 9, 2019.

\bibitem[Smid et~al.(2010)Smid, Leigh, Levenson, Vangel, DavidBanks, and JamesDray]{smid2010statistical}
Elaine~Barker Smid, Stefan Leigh, Mark Levenson, Mark Vangel, A~DavidBanks, and SanVo JamesDray.
\newblock A statistical test suite for random and pseudorandom number generators for cryptographic applications.
\newblock \emph{Her research interest includes Computer security, secure operating systems, Access control, Distributed systems, Intrusion detection systems}, 2010.

\bibitem[Tao et~al.(2025)Tao, Doshi, Kalra, He, and Barkeshli]{tao2025can}
Tao Tao, Darshil Doshi, Dayal~Singh Kalra, Tianyu He, and Maissam Barkeshli.
\newblock (how) can transformers predict pseudo-random numbers?
\newblock \emph{arXiv preprint arXiv:2502.10390}, 2025.

\bibitem[Vaswani(2017)]{vaswani2017attention}
A~Vaswani.
\newblock Attention is all you need.
\newblock \emph{Advances in Neural Information Processing Systems}, 2017.

\bibitem[Von~Oswald et~al.(2023)Von~Oswald, Niklasson, Randazzo, Sacramento, Mordvintsev, Zhmoginov, and Vladymyrov]{von2023transformers}
Johannes Von~Oswald, Eyvind Niklasson, Ettore Randazzo, Jo{\~a}o Sacramento, Alexander Mordvintsev, Andrey Zhmoginov, and Max Vladymyrov.
\newblock Transformers learn in-context by gradient descent.
\newblock In \emph{International Conference on Machine Learning}, pages 35151--35174. PMLR, 2023.

\bibitem[Wei et~al.(2022)Wei, Wang, Schuurmans, Bosma, Xia, Chi, Le, Zhou, et~al.]{wei2022chain}
Jason Wei, Xuezhi Wang, Dale Schuurmans, Maarten Bosma, Fei Xia, Ed~Chi, Quoc~V Le, Denny Zhou, et~al.
\newblock Chain-of-thought prompting elicits reasoning in large language models.
\newblock \emph{Advances in neural information processing systems}, 35:\penalty0 24824--24837, 2022.

\bibitem[Wu et~al.(2025{\natexlab{a}})Wu, Han, Zhang, Li, and Cui]{wu2025gan}
Xuguang Wu, Yiliang Han, Minqing Zhang, Yu~Li, and Su~Cui.
\newblock Gan-based pseudo random number generation optimized through genetic algorithms.
\newblock \emph{Complex \& Intelligent Systems}, 11\penalty0 (1):\penalty0 31, 2025{\natexlab{a}}.

\bibitem[Wu et~al.(2025{\natexlab{b}})Wu, Han, Zhang, Zhu, Cui, Wang, and Peng]{wu2025pseudorandom}
Xuguang Wu, Yiliang Han, Minqing Zhang, ShuaiShuai Zhu, Su~Cui, Yuanyuan Wang, and Yixuan Peng.
\newblock Pseudorandom number generators based on neural networks: a review.
\newblock \emph{Journal of King Saud University Computer and Information Sciences}, 37\penalty0 (3):\penalty0 1--27, 2025{\natexlab{b}}.

\bibitem[Yang et~al.(2024)Yang, Ackermann, He, Feng, Zhang, Feng, Ye, He, and Wang]{yang2024efficient}
Kai Yang, Jan Ackermann, Zhenyu He, Guhao Feng, Bohang Zhang, Yunzhen Feng, Qiwei Ye, Di~He, and Liwei Wang.
\newblock Do efficient transformers really save computation?
\newblock \emph{arXiv preprint arXiv:2402.13934}, 2024.

\bibitem[Yun et~al.(2019)Yun, Bhojanapalli, Rawat, Reddi, and Kumar]{yun2019transformers}
Chulhee Yun, Srinadh Bhojanapalli, Ankit~Singh Rawat, Sashank~J Reddi, and Sanjiv Kumar.
\newblock Are transformers universal approximators of sequence-to-sequence functions?
\newblock \emph{arXiv preprint arXiv:1912.10077}, 2019.

\bibitem[Yun et~al.(2020)Yun, Chang, Bhojanapalli, Rawat, Reddi, and Kumar]{yun2020n}
Chulhee Yun, Yin-Wen Chang, Srinadh Bhojanapalli, Ankit~Singh Rawat, Sashank Reddi, and Sanjiv Kumar.
\newblock O (n) connections are expressive enough: Universal approximability of sparse transformers.
\newblock \emph{Advances in Neural Information Processing Systems}, 33:\penalty0 13783--13794, 2020.

\end{thebibliography}

\appendix  
\section*{Appendix}
\label{sec:appendix}

\numberwithin{theorem}{section}  

The appendix is organized as follows: 
\begin{itemize}
    \item  \Cref{sec:log_precision}, we introduce the log-precision parameters of the Transformer. 
    \item  \Cref{sec:technical_lemmas}, we use FFN and Transformer to implement
some basic operations in the linear congruential generator algorithm and the Mersenne Twister algorithm.
    \item  \Cref{sec:proof_expressiveness}, we prove the main theorem of this paper, 
the expressiveness of the Transformer.
\end{itemize}

\section{Log-precision}  
\label{sec:log_precision}

In real computers, the parameters of transformers are 
finite precision floating-point numbers. Therefore, 
when conducting theoretical analysis, we need to adopt assumptions 
that align with practical situations.
There are two most common assumptions to store real numbers:
the fixed-point format and floating-point format.
Also, there are many methods to truncate the number to a certain precision.
such as round-to-the-nearest, round-to-zero, round-up, and round-down.

Our assumption is that the number is log-precision, which means
the number is represented by \( O(\log n) \) bits,
where \( n \) is the maximum length of the input sequence.
Our theoretical results also hold for all above truncation methods.
The analysis of truncation error is similar as \cite{feng2024towards}.
The function represented by transformers are continuous.
Therefore, the approximation error in a hidden
neuron will smoothly influence the next neuron.
This impact is bounded by the Lipschitz constant of the transformer
which depends on its architecture.
In particular, the softmax function is \( 1 \)-Lipschitz,
the GeLU function is \( 2 \)-Lipschitz,
and the Lipschitz constant of the Linear layer is 
determined by its weight parameters.
Combining these Lipschitz constants leads to results:
given a bound-depth log-precision Transformer of polynomial size,
when all parameters are bounded by \( O(poly(n)) \),
the truncation error is \( O(poly(\frac{1}{n})) \).
So if problem can be solved by a infinite precision Transformer of polynomial size,
it can also be solved by a log-precision Transformer of polynomial size.

\section{Technical Lemmas}
\label{sec:technical_lemmas}

In this section, we prove some technical lemmas for the MLP and Transformer.
Our proof is similar to Feng et al. \cite{feng2024towards}.
So we will use some lemmas 
in Feng et al. \cite{feng2024towards} and Yang et al. \cite{yang2024efficient}
without proof.

\subsection{Technical Lemmas for MLP}

First, we use MLP to approximate the multiplication function.

\begin{lemma}{(Lemma C.1 in \cite{feng2024towards})}
    \label{lemma:multiplication}
    Let \( f : \mathbb{R}^2 \to \mathbb{R} \) 
    be a two-layer MLP with GeLU activation, and the hidden dimension is 4. 
    Then, for any \( \epsilon > 0 \) and \( M > 0 \), 
    there exist MLP parameters with \( \ell_\infty \) norm upper bounded by 
    \( O(poly(M, 1/\epsilon)) \) such that 
    \[
    |f(a, b) - ab| \leq \epsilon
    \]
    holds for all \( a, b \in [-M, M] \).
\end{lemma}

Next, we use MLP to approximate the ReLU activation function.

\begin{lemma}{(Lemma C.2 in \cite{feng2024towards})}
    \label{lemma:relu}
    Let \( g : \mathbb{R}^{d_1} \to \mathbb{R}^{d_2} \) 
    be a two-layer MLP with ReLU activation, 
    and all parameter values are upper bounded by \( M \). 
    Then, for any \( \epsilon > 0 \), 
    there exists a two-layer MLP \( f \) of the same size with GeLU activation 
    and parameters upper bounded by \( O(poly(M, 1/\epsilon)) \) in the \( \ell_\infty \) norm, 
    such that for all \( x \in \mathbb{R}^{d_1} \), we have 
    $\| f(x) - g(x) \|_\infty \leq \epsilon.$
\end{lemma}

We use MLP to approximate the selection function.
\begin{lemma}{(Lemma C.4 in \cite{feng2024towards})}
    \label{lemma:selection}
    Define the selection function \( g : \mathbb{R}^d \times \mathbb{R}^d \times \mathbb{R} \to \mathbb{R}^d \) as follows:
    \[
    g(x, y, t) = 
    \begin{cases} 
    x & \text{if } t \geq 0, \\ 
    y & \text{if } t < 0. 
    \end{cases}
    \]
    
    Let \( f : \mathbb{R}^d \times \mathbb{R}^d \times \mathbb{R} \to \mathbb{R} \) be a two-layer MLP with GeLU activation, 
    and the hidden dimension is \( 2d + 2 \). 
    Then, for any \( \epsilon > 0, \alpha > 0, \) and \( M > 0 \), 
    there exist MLP parameters with \( \ell_\infty \) norm bounded by \( O(poly(M, 1/\alpha, 1/\epsilon)) \), 
    such that for all \( x \in [-M, M]^d, y \in [-M, M]^d, \) and \( t \in [-\infty, -\alpha] \cup [\alpha, +\infty] \), 
    we have 
    $\| f(x, y, t) - g(x, y, t) \|_\infty \leq \epsilon.$
\end{lemma}

We use a two-layer MLP to approximate the modulus function.
First, we can rewrite the modulus operation as $ x \mod y = x - y \lfloor \frac{x}{y} \rfloor $.
Then we prove that an MLP can approximate the function $f(x, y) = x - y \lfloor \frac{x}{y} \rfloor$.

\begin{lemma}
    \label{lemma:modulus}
    Given positive number \( n \), and integer \( i \in [1, n^2] \),
    There exists a two-layer MLP \( f : \mathbb{R}^2 \to \mathbb{R} \) 
    with GeLU activation, and the hidden dimension is \( O( n ) \).
    Then exists a constant \( \epsilon > 0 \) such that 
    \( \|f(i, n) - (i \bmod n)\|_\infty \leq \epsilon \) and 
    parameters with \( \ell_\infty \) norm upper bounded by \( O(poly(M, 1/\epsilon)) \),
    where \( M \) is an upper bound for all parameter values.
\end{lemma}

\begin{proof}
    The modulus operation can be rewritten as $ x \bmod y = x - y \lfloor \frac{x}{y} \rfloor $.
    Given that \( i \in [1, n^2] \), we know \( \lfloor \frac{i}{n} \rfloor \in [1, n] \).
    Following \cite{yang2024efficient}Lemma A.7, we can construct an MLP with hidden dimension \( O( n ) \) to
    get \(\lfloor \frac{i}{n} \rfloor\) as follows:
    \begin{align*}
    \lfloor \frac{i}{n} \rfloor 
    &= \left(\sum_{j=1}^{n} \mathbb{I}[i \leq jn]\right) - 1 \\
    &= \left(\sum_{j=1}^{n} ReLU[- 2i + 2 (jn + \frac{3}{4})] - ReLU[- 2i + 2 (jn + \frac{1}{4})]\right) - 1.
    \end{align*}
    Then we can calculate \( i - n \lfloor \frac{i}{n} \rfloor \) by an MLP with ReLU activation 
    with hidden dimension \( O( n ) \) and 
    parameters with \( \ell_\infty \) norm upper bounded by \( O(poly(M, 1/\epsilon)) \) according to \cite{yang2024efficient}Lemma A.7.
\end{proof}

\subsection{Technical Lemmas for Transformer}

In this section, we prove one attention layer can 
approximate boolean function.Such as AND, OR, NOT, XOR operation.
Before prove, we need some definitions.
Let's \( x_1, x_2, \cdots, x_d \) be a variable vector.
\( x_i = ( \phi(i), i, 1) \), where \( \phi(i) \in \{0, 1\}^d \) is a boolean vector,
\( i \in [n] \) is a index variable, 1 is a constant value.
First, we prove constant AND variable, constant OR variable, constant XOR variable 
can be approximated by one attention layer.
Then,  combine these operations and MLP lemma, we can prove 
variable AND variable, variable OR variable, 
variable XOR variable can be approximated by Transformer layer.

\subsubsection{Boolean Operations between Constants and Variables}

\begin{lemma}
    \label{lemma:constant_and_variable}
    Given a constant vector \( \mathbf{c} \in \{0, 1\}^d \) 
    and a variable vector \( \mathbf{x}_i = ( \phi(i), i, 1) \),
    there exists a attention layer that can approximate 
    the boolean function \( \mathbf{c} \land \phi(i) \).
\end{lemma}

\begin{proof}
    Because the AND operation with constants in Boolean algebra 
    is consistent with matrix multiplication, 
    we can construct the \( \mathbf{W}_V \) matrix as follows.
    \[
    \mathbf{W}_V = \begin{pmatrix}
    n*c_1 & 0 & \cdots & 0  & 0 & 0\\
    0 & n*c_2 & \cdots & 0 & 0 & 0\\
    \vdots & \vdots & \ddots & \vdots & \vdots & \vdots\\
    0 & 0 & \cdots & n*c_d  & 0 & 0\\
    0 & 0 & \cdots & 0 & 1 & 0\\
    0 & 0 & \cdots & 0 & 0 & 1\\
    \end{pmatrix}
    \]
    Then setting \( \mathbf{W}_Q = \mathbf{W}_K =  \mathbf{0}_{(d+2) \times (d+2)} \) and \( \mathbf{W}_O = I \),
    we can get the boolean function \(  \mathbf{c} \land \phi(i) = softmax((\mathbf{X} \mathbf{W}_Q)(\mathbf{X} \mathbf{W}_K)^T)\mathbf{X} \mathbf{W}_V \mathbf{W}_O \).
\end{proof}

\begin{corollary}
    \label{corollary:constant_right_left_shift_variable}
    Given a constant number \( c \in \mathbb{N} \)
    and a variable vector \( \mathbf{x}_i = ( \phi(i), i, 1) \),
    there exists a attention layer that can approximate 
    the boolean function \( \phi(i) \ll c \) and \( \phi(i) \gg c \).
\end{corollary}

\begin{proof}
    Because \( \phi(i) \ll c = \phi(i) \land (c \ll 1) \) and \( \phi(i) \gg c = \phi(i) \land (c \gg 1) \),
    so we can use the~\Cref{lemma:constant_and_variable} to prove this corollary.
\end{proof}

\begin{lemma}
    \label{lemma:not_variable}
    Given a variable vector \( \mathbf{x}_i = ( \phi(i), i, 1) \),
    there exists a attention layer that can approximate 
        the boolean function \( \neg \phi(i) \).
\end{lemma}

\begin{proof}
    Because \( \neg \mathbf{x}_i = 1 - \mathbf{x}_i \),so 
    we can construct the \( \mathbf{W}_V \) matrix as follows.
    \[
    \mathbf{W}_V = \begin{pmatrix}
    -n & 0 & \cdots & 0 & 0 & 0\\
    0 & -n & \cdots & 0 & 0 & 0\\
    \vdots & \vdots & \ddots & \vdots & \vdots & \vdots\\
    0 & 0 & \cdots & -n & 0 & 0\\
    0 & 0 & \cdots & 0 & 1 & 0\\
    n & n & \cdots & n & 0 & 1\\
    \end{pmatrix}
    \]
    Then setting \( \mathbf{W}_Q = \mathbf{W}_K =  \mathbf{0}_{(d+2) \times (d+2)} \) and \( \mathbf{W}_O = I \),
    we can get the boolean function \( \neg \phi(i) = softmax((\mathbf{X} \mathbf{W}_Q)(\mathbf{X} \mathbf{W}_K)^T)\mathbf{X} \mathbf{W}_V \mathbf{W}_O \).
\end{proof}

\begin{lemma}
    \label{lemma:constant_or_variable}
    Given a constant vector \( \mathbf{c} \in \{0, 1\}^d \) 
    and a variable vector \( \mathbf{x}_i = ( \phi(i), i, 1) \),
    there exists a attention layer that can approximate 
    the boolean function \( \mathbf{c} \lor \phi(i) \).
\end{lemma}

\begin{proof}
    Because \( \mathbf{c} \lor \phi(i) = \neg (\neg \mathbf{c} \land \neg \phi(i)) \),
    so we can use the~\Cref{lemma:not_variable} and~\Cref{lemma:constant_and_variable} to construct special attention layer
    to prove this lemma.
\end{proof}

\begin{lemma}
    \label{lemma:constant_xor_variable}
    Given a constant vector \( \mathbf{c} \in \{0, 1\}^d \) 
    and a variable vector \( \mathbf{x}_i = ( \phi(i), i, 1) \),
    there exists a attention layer that can approximate 
    the boolean function \( \mathbf{c} \oplus \phi(i) \).
\end{lemma}

\begin{proof}
    Because \( \mathbf{c} \oplus \phi(i) = (\mathbf{c} \land \neg \phi(i)) \lor (\neg \mathbf{c} \land \phi(i)) \),
    so we can use the~\Cref{lemma:constant_and_variable}, \Cref{lemma:not_variable}, and~\Cref{lemma:constant_or_variable} 
    to construct special attention layer to prove this lemma.
\end{proof}

By combining the lemma of the MLP and the lemma of attention, 
we can achieve Boolean operations between variables.

\subsubsection{Boolean Operations between Variables}

\begin{lemma}
    \label{lemma:variable_and_variable}
    For a given variable vector \( ( \phi(i), \phi(j), i, 1 ) \),
    there exists a Transformer layer that approximates 
    function \( f(i, j) \) with a constant \( \epsilon > 0 \) such that 
    \( \|f(i, j) - \phi(i) \land \phi(j)\|_\infty \leq \epsilon \).
    The parameters of this layer exhibit an \( \ell_\infty \) norm bounded by 
    \( O(poly(M, 1/\epsilon)) \), where M denotes the upper bound for all parameter values.
\end{lemma}

\begin{proof}
    Because \( \phi(i) \land \phi(j) = ReLU(\phi(i) + \phi(j) - 1) \),
    We construct the attention matrix \( \mathbf{V} \) to get \( \phi(i) + \phi(j) -1 \) as follows.
    \[
    \mathbf{V} = \begin{pmatrix}
    1 & 0 & \cdots & 0 & \vline & 1\\
    0 & 1 & \cdots & 0 & \vline & 1\\
    \vdots & \vdots & \ddots & \vdots & \vline & \vdots\\
    0 & 0 & \cdots & 1 & \vline & 0\\
    0 & 0 & \cdots & 0 & \vline & -1\\
    \end{pmatrix}
    \]
    So \( ( \phi(i), \phi(j), i, 1 ) \mathbf{V} = ( \phi(i), \phi(j), i, 1, \phi(i) + \phi(j) - 1 ) \).
    Due to~\Cref{lemma:relu}, we can use a two-layer MLP to get
    \( ReLU(\phi(i) + \phi(j) - 1) \) and 
    parameters with \( \ell_\infty \) norm upper bounded by \( O(poly(M, 1/\epsilon)) \).
\end{proof}

\begin{lemma}
    \label{lemma:variable_or_variable}
    For a given variable vector \( ( \phi(i), \phi(j), i, 1 ) \),
    there exists a Transformer layer that approximates 
    function \( f(i, j) \) with a constant \( \epsilon > 0 \) such that 
    \( \|f(i, j) - \phi(i) \lor \phi(j)\|_\infty \leq \epsilon \).
    The parameters of this layer exhibit an \( \ell_\infty \) norm bounded by 
    \( O(poly(M, 1/\epsilon)) \), where M denotes the upper bound for all parameter values.
\end{lemma}

\begin{proof}
    Because \( \phi(i) \lor \phi(j) =g((-\phi(i) - \phi(j) + 1), 1, 0) \),
    \[
    g(x, 1, 0) = 
    \begin{cases}
    x & \text{if } x \leq 0, \\
    1 & \text{if } x > 0.
    \end{cases}
    \]
    First, we construct the attention matrix \( \mathbf{V} \) to get \( -\phi(i) - \phi(j) + 1 \) as follows.
    \[
    \mathbf{V} = \begin{pmatrix}
    1 & 0 & \cdots & 0 & \vline & -1\\
    0 & 1 & \cdots & 0 & \vline & -1\\
    \vdots & \vdots & \ddots & \vdots & \vline & \vdots\\
    0 & 0 & \cdots & 1 & \vline & 0\\
    0 & 0 & \cdots & 0 & \vline & 1\\
    \end{pmatrix}
    \]
    So \( ( \phi(i), \phi(j), i, 1 ) \mathbf{V} = ( \phi(i), \phi(j), i, 1, -\phi(i) - \phi(j) + 1 ) \).
    Due to~\Cref{lemma:selection}, we can use a two-layer MLP to get
    \( g(-\phi(i) - \phi(j) + 1, 1, 0) \) and 
    parameters with \( \ell_\infty \) norm upper bounded by \( O(poly(M, 1/\epsilon)) \).
\end{proof}

\begin{lemma}
    \label{lemma:variable_xor_variable}
    For a given variable vector \( ( \phi(i), \phi(j), i, 1 ) \),
    there exists a Transformer layer that approximates 
    function \( f(i, j) \) with a constant \( \epsilon > 0 \) such that 
    \( \|f(i, j) - \phi(i) \oplus \phi(j)\|_\infty \leq \epsilon \).
    The parameters of this layer exhibit an \( \ell_\infty \) norm bounded by 
    \( O(poly(M, 1/\epsilon)) \), where M denotes the upper bound for all parameter values.
\end{lemma}

\begin{proof}
    Because \( \phi(i) \oplus \phi(j) = (\phi(i) - \phi(j))^2 \),
    First, we construct the attention matrix \( \mathbf{V} \) to get \( \phi(i) - \phi(j) \) as follows.
    \[
    \mathbf{V} = \begin{pmatrix}
    I_{\phi(i) \times \phi(i)} & 0 & 0 & 0 &\vline & \mathbf{1}_{\phi(i)}\\
    0 & I_{\phi(j) \times \phi(j)} & 0 & 0 &\vline & -\mathbf{1}_{\phi(j)}\\
    0 & 0 & 1 & 0 &\vline & 0\\
    0 & 0 & 0 & 1 &\vline & 0\\
    \end{pmatrix}
    \]
    So \( ( \phi(i), \phi(j), i, 1 ) \mathbf{V} = ( \phi(i), \phi(j), i, 1, (\phi(i) - \phi(j))^2 ) \).
    Due to~\Cref{lemma:multiplication}, we can use a two-layer MLP to get
    \( (\phi(i) - \phi(j))^2 \) and
    parameters with \( \ell_\infty \) norm upper bounded by \( O(poly(M, 1/\epsilon)) \).
\end{proof}

\section{Missing proof in ~\Cref{sec:expressiveness}}
\label{sec:proof_expressiveness}

In this section, we will prove that the Transformer can simulate 
the linear congruential generator algorithm and the Mersenne Twister algorithm.

Before proving the theorems, we use same assumptions as 
in ~\cite{feng2024towards}: all residual connections in the 
attention layers can be replaced by concatenation, in the sense that
both architectures have the same expressive power.
Consider a \( \mathbf{x},f(\mathbf{x}) \in \mathbb{R}^{d},\), 
First, concatenation can implement the residual connection by 
using linear projections. \( g(\mathbf{x}) = f(\mathbf{x}) + \mathbf{x} \).
Conversely, the residual connection can be implement concatenation
by using another MLP, for example, \( h(\mathbf{x}, \mathbf{0}) + (\mathbf{x}, \mathbf{0}) = (f(\mathbf{x}), \mathbf{x}) \).
So in the following proof, we will use the concatenation operation
to replace the residual connection.

\subsection{Proof of \Cref{thm:LCG}}
\label{proof:LCG}

\begin{theorem}
    For the linear congruential generator algorithm, 
    we can construct an autoregressive Transformer as defined 
    in ~\Cref{sec:preliminaries} that is capable of 
    simulating and generating $n$ pseudo-random numbers. 
    This Transformer has a hidden dimension of $d = O(n)$, 
    uses a fixed number of layers and attention heads per layer, 
    and all of its parameters are polynomially bounded by $O(poly(n))$.
\end{theorem}

\begin{proof}
    A single layer Transformer is used to simulate the linear congruential generator algorithm.

   \textbf{Token embeddings}. Consider a sequence of tokens $\mathbf{s}_1, \dots , \mathbf{s}_i$, 
   with the goal of generating $\mathbf{s}_{i+1}$. Each $\mathbf{s}_j \in \mathbb{N}^d$ is 
   a $d$-dimensional binary vector. The token $\mathbf{s}_j$ can be embedded as:
   $\mathbf{x}_{j}^{(0)} = (\mathbf{s}_j,j,1) \in \mathbb{R}^{d+2}$. 
   Here, $j \in \mathbb{N}^+$ serves as the positional embedding, 
   and the constant embedding acts as a bias term.

   \textbf{Layer 1}. The first layer of the Transformer has two tasks:
   \begin{enumerate}
    \item Compute $\mathbf{x}_i^{(0)}\mathbf{A} + \mathbf{C}$.
    \item Compute $(\mathbf{x}_i^{(0)}\mathbf{A} + \mathbf{C}) \% \mathbf{N}$.
    where $\mathbf{A} \in \mathbb{R}^{d \times d}$, $\mathbf{C} \in \mathbb{R}^{d}$, and $\mathbf{N} \in \mathbb{R}$.
   \end{enumerate}
   A special matrix $\mathbf{V}$ can be constructed as follows:
   \[
    \mathbf{V} = \begin{bmatrix}
        a_1 & 0 & \cdots & 0 & 0 & 0\\
        0 & a_2 & \cdots & 0 & 0 & 0\\
        \vdots & \vdots & \ddots & \vdots & \vdots & \vdots\\
        0 & 0 & \cdots & a_d & 0 & 0\\
        0 & 0 & \cdots & 0 & 1 & 0\\
        c_1 & c_2 & \cdots & c_d & 0 & 1\\
    \end{bmatrix}
   \]
   Thus, $(\mathbf{s}_j,j,1)\mathbf{V} = (\mathbf{s}_j\mathbf{A} + \mathbf{C},j,1)$.
   According to~\Cref{lemma:modulus}, an MLP can be used to compute $\mathbf{s}_j\mathbf{A} + \mathbf{C}$ and $(\mathbf{s}_j\mathbf{A} + \mathbf{C}) \% \mathbf{N}$, yielding the next pseudo-random number $\mathbf{s}_{j+1}$.
   
   Next, we analyze the parameters of the Transformer.Because $\mathbf{V}$ is a constant matrix and 
   based on~\Cref{lemma:modulus}. Therefore, by picking fixed small $\epsilon = \Theta(1)$,
   the parameters of the Transformer are polynomially bounded by $O(poly(n))$.
\end{proof}

\subsection{Proof of Theorem 6}
\label{proof:MT}
\begin{theorem}
    For the Mersenne Twister algorithm, 
    we can construct an autoregressive Transformer as defined 
    in ~\Cref{sec:preliminaries} that is capable of 
    simulating and generating $n$ pseudo-random numbers. 
    This Transformer has a hidden dimension of $d = O(n)$, 
    uses a fixed number of layers and attention heads per layer, 
    and all of its parameters are polynomially bounded by $O(poly(n))$.
\end{theorem}

\begin{proof}
    We construct each layer as follows.

    \textbf{Token Embeddings}. Given a sequence of tokens $s_1, \dots , s_i$, we aim to generate $s_{i+1}$. For any $a \in \mathbb{N}$, let $\phi^w(a)$ represent the $w$-bit binary representation of $a$, where $\phi^w(a) \in \{0,1\}^w$ and $w \in \mathbb{N}^+$. Hereafter, we use $\phi(a)$ to denote $\phi^w(a)$. The token $s_j$ can be embedded as $\mathbf{x}^{(0)}_j = (\phi(s_j),j,1) \in \mathbb{R}^{w+2}$, where $j \in \mathbb{N}^+$ serves as the positional embedding, and the constant embedding acts as a bias term. If $s_j = ` \Rightarrow \text{'}$, then $\phi(`\Rightarrow \text{'}) = \underbrace{0,\dots,0}_\text{$w$-bits} $.

    \textbf{Layer 1}.
    The initial layer of the autoregressive Transformer is equipped with a single attention head, which is responsible for the following operations:
    \begin{enumerate}
        \item Calculation of the count of right arrow symbols $` \Rightarrow \text{'}$ in the preceding tokens, denoted as $cnt_{\Rightarrow}$, i.e., $cnt_{\Rightarrow} = |\{ j \le i : s_j = ` \Rightarrow \text{'} \}|$.
        \item Computation of three indices $t\%n,(t+1)\%n,(t+m)\%n$ for the generation of the $t$-th random number, i.e., $t = (cnt_{\Rightarrow}-1)\%n, t+1 = (t + 1) \% n, t+m = (t + m) \% n$. 
    \end{enumerate}
    Given $cnt_{\Rightarrow} = \lceil \frac{i - (n + 1)}{n + 2} \rceil$ and $t = (cnt_{\Rightarrow}-1)\%n$, a specialized matrix $V$ can be configured to compute $\frac{i - (n + 1)}{n + 2}, \frac{i - (n + 1)}{n + 2} -1, \frac{i - (n + 1)}{n + 2} + m -1$. For instance,
    \begin{equation*}
         (\phi(s_i),i,1)*
         \begin{bmatrix}
         \multirow{3}{*}{\makebox[4em]{$\mathbf{I}_{w \times w}$}}  & \vline & \Vec{\mathbf{0}} & \Vec{\mathbf{0}} & \Vec{\mathbf{0}} & \Vec{\mathbf{0}} \\
         & \vline & \frac{1}{n+2} & \frac{1}{n+2} & \frac{1}{n+2} & \frac{1}{n+2} \\
         & \vline & \frac{1}{n+2} & -\frac{n+1}{n+2} & \frac{1}{n+2} & \frac{1}{n+2} + m - 1 \\
         \end{bmatrix}
    \end{equation*}
    Here, $\mathbf{I}_{w \times w}$ denotes the $w \times w$ identity matrix, and $\mathbf{0}$ represents the $w$-bit zero vector. By performing matrix multiplication with $V$, the floating point representation of $cnt_{\Rightarrow},t\%n,(t+1)\%n,(t+m)\%n$ can be obtained. Subsequently, an MLP can be utilized to truncate and apply modulus $n$ to $t,t+1,t+m$, all of which can be achieved through MLP operations (\Cref{lemma:modulus}).
    The ultimate output of the initial layer is in the form of 
    $\mathbf{x}_i^{(1)} = (\phi(s_i),i,1,cnt_{\Rightarrow},t\%n,(t+1)\%n,(t+m)\%n)$.
\\
\\
    \textbf{Layer 2}.
     The second layer of the Transformer performs some intricate preparatory work for the next task.
    \begin{enumerate}
        \item Calculate the positions of the nearest and the last right arrow symbols `$\Rightarrow$', denoted as $now_{\Rightarrow}$ and $pre_{\Rightarrow}$, respectively. Note that $pre_{\Rightarrow}$ is initialized to $-1$.
        \item Calculate the actual coordinates of the three indices $t\%n,(t+1)\%n$, and $(t+m)\%n$ corresponding to the sequence, denoted as $now_t,now_{t+1}$, and $now_{t+m}$, respectively.
        \item Calculate $judge$ and $next_x$. $judge$ is used to determine whether the next token $s_{i+1}$ is a random number, an initial number, or `$\Rightarrow$', and $next_x$ is the actual coordinate of the next initial number.
        \item Determine whether the current position $i$ is equal to $now_t,now_{t+1},now_{t+m}$, or $next_x$, and indicate it with four indicators $ind_{now_t},ind_{now_{t+1}},ind_{now_{t+m}}$, and $ind_{next_x}$. For example, if $i = now_t$, then $ind_{now_t} = 0$; otherwise, $ind_{now_t} = -\infty$.
    \end{enumerate}
    Because $now_{\Rightarrow} = (n + 1) + (n + 2) * (cnt_{\Rightarrow} - 1)$, 
    $pre_{\Rightarrow} = (n + 1) + (n + 2) * (cnt_{\Rightarrow} - 2)$, 
    $now_t = pre_{\Rightarrow} + t\%n +2$,
    $now_{t+1} = pre_{\Rightarrow} + (t+1)\%n +2$,
    $now_{t+m} = pre_{\Rightarrow} + (t+m)\%n +2$,
    $judge = i - now_{\Rightarrow}$, 
    and $next_x = pre_{\Rightarrow} + judge$,
    so we can construct special matrix $\mathbf{V}$ as follow:
    \begin{equation*}
    \scalebox{0.85}{%
         $\mathbf{V} = 
         \begin{bmatrix}
         \multirow{7}{*}{\makebox[5em]{$\mathbf{I}_{(w+6) \times (w+6)}$}}  & \vline & \mathbf{0} & \mathbf{0} & \mathbf{0} & \mathbf{0} & \mathbf{0} & \mathbf{0} & \mathbf{0} \\
         & \vline & 0 & 0 & 0 & 0 & 0 & 1 & 1 \\
         & \vline & -1 & -(n+3) & -(n+1) & -(n+1) & -(n+1) & -1 & -(n+4) \\
         & \vline & n+2 & n+2 & n+2 & n+2 & n+2 & n+2 & 2*(n+2) \\
         & \vline & 0 & 0 & 1 & 0 & 0 & 0 & 0 \\
         & \vline & 0 & 0 & 0 & 1 & 0 & 0 & 0 \\
         & \vline & 0 & 0 & 0 & 0 & 1 & 0 & 0 \\
         \end{bmatrix}$%
    }
    \end{equation*}
    We then use MLP to simulate a function $f(x,y)$ as follow:
    \begin{equation*}
        f(x,y) = \begin{cases}
            0 & \text{if } x = y \\
            -\infty & \text{otherwise}
        \end{cases}
    \end{equation*}
    which is a conditional selection operation and can be implemented by an MLP (\Cref{lemma:selection}).
    we calculate \(f(i,now_t) \), \(f(i,now_{t+1}) \), \(f(i,now_{t+m}) \), \(f(i,next_x)\)
    to get \( ind_{now_t} \), \(ind_{now_{t+1}} \), \(ind_{now_{t+m}} \), \(ind_{next_x}\). The output of the second layer is 
    \begin{equation*}
        \scalebox{0.76}{%
            $\mathbf{x}_i^{(2)} = (\phi(s_i), i, 1, now_{\Rightarrow}, pre_{\Rightarrow}, 
                        now_t, now_{t+1}, now_{t+m}, next_x, judge, 
                        ind_{now_t}, ind_{now_{t+1}}, ind_{now_{t+m}}, ind_{next_x})$%
        }
    \end{equation*}
    and we omitted some variables $cnt_{\Rightarrow},t\%n,(t+1)\%n,(t+m)\%n$.
\\
\\
    \textbf{Layer 3}. The third layer of the Transformer is responsible for extracting the three vectors used by the Mersenne Twister algorithm, as well as the vector related to generating the next token. This is achieved through the use of five attention heads, each performing a specific task:
    \begin{enumerate}
        \item Extraction of four $w$-bit vectors and their addition to $\mathbf{x}_i^{(3)}$. These vectors are $\phi(s_{now_t}),\phi(s_{now_{t+1}})$,$\phi(s_{now_{t+m}}),\phi(s_{next_x})$.
    \end{enumerate}
    The extraction process of $\phi(s_{now_t})$ is used as an illustration. Special matrices $\mathbf{Q}$ and $\mathbf{K}$ are constructed to obtain the bias and indicators for all $j \leq i$, forming an $i \times 1$-dimensional vector. For example:
    \begin{equation*}
        \mathbf{Q} = \mathbf{X}\mathbf{W}_Q = 
        \begin{bmatrix}
            1 \\
            1 \\
            \vdots \\
            1 \\
        \end{bmatrix}
        \qquad
        \mathbf{K}^T = (\mathbf{X}\mathbf{W}_K)^T =
        \underbrace{
        \begin{bmatrix}
            -\infty,& 0,& \ldots,& -\infty \\
        \end{bmatrix}^T        
        }_\text{$ind_{now_t}$ for all $j \le i$}
    \end{equation*}
    Let $\mathbf{W}_O = \mathbf{I},\mathbf{W}_V = \mathbf{I}$, and through the operation $\mathbf{Att}^{(3)}\mathbf{X} = \text{softmax}(\mathbf{Q}\mathbf{K}^T)\mathbf{X}\mathbf{W}_V\mathbf{W}_Q$, we can let $\mathbf{x}_i^{(3)}$ obtain information of $\phi(s_{now_t})$. The output of $\mathbf{Att}^{(3)}\mathbf{X}$ can be represented as:
    \begin{equation*}
        \begin{bmatrix}
            \phi(s_{now_t}) & 1 & 1 & \ldots & ind_{next_x} \\
            \phi(s_{now_t}) & 2 & 1 & \ldots & ind_{next_x} \\
            \vdots & \vdots & 1 & \ldots & ind_{next_x} \\
            \phi(s_{now_t}) & i & 1 & \ldots & ind_{next_x} \\
        \end{bmatrix}
    \end{equation*}
    By using the residual connection to perform concatenation, we can add $\phi(s_{now_t})$ to $\mathbf{x}_i^{(3)}$. Similarly, the other three attention heads can be used to add $\phi(s_{now_{t+1}})$, $\phi(s_{now_{t+m}})$, $\phi(s_{next_x})$ to $\mathbf{x}_i^{(3)}$. The final output of the third layer is represented as:
    \begin{equation*}
        \scalebox{0.9}{%
            $\mathbf{x}_i^{(3)} = (\phi(s_i), i, 1, now_{\Rightarrow}, pre_{\Rightarrow}, judge, \ldots, \phi(s_{now_t}),\phi(s_{now_{t+1}}),\phi(s_{now_{t+m}}),\phi(s_{next_x}))$%
        }        
    \end{equation*}
    The variables that have been omitted are represented by `$\ldots$'.
\\
\\    
    We can now simulate the process of generating random numbers using the Mersenne Twister algorithm, which is divided into two blocks: rotation and extraction.\\   
    \textbf{Block 1 (Rotation)}. In the first block, we simulate the Mersenne rotation process:
    \begin{equation*}
        \mathbf{t} \gets (\mathbf{x}[i] \, \land \, \mathbf{upper}) \, \lor \, (\mathbf{x}[(i+1)\, \bmod \, n] \, \land \, \mathbf{lower})
    \end{equation*}
    \begin{equation*}
        \mathbf{z} \gets \mathbf{x}[(i+m)\, \bmod \, n] \, \oplus \, (\mathbf{t} >> 1) \, \oplus \, \begin{cases}
            \mathbf{0} & \text{if } t_0 = 0 \\
            \mathbf{a} & \text{otherwise}
        \end{cases}
    \end{equation*}
    Here, $\mathbf{upper}$, $\mathbf{lower}$, and $\mathbf{a}$ are constants.
    
    \begin{itemize}
        \item \textbf{Layer 4}. We can construct a special matrix $V$ to compute $\mathbf{x}[(i+m)\, \bmod \, n]$ 
        and $\mathbf{x}[(i+1)\, \bmod \, n] \, \land \, \mathbf{lower}$, which can be implemented using~\Cref{lemma:modulus}.
        \item \textbf{Layer 5 - 6}. In the fifth and sixth layers, we can calculate
        $(\mathbf{x}[i] \, \land \, \mathbf{upper}) \, \lor \, (\mathbf{x}[(i+1)\, \bmod \, n] \, \land \, \mathbf{lower})$ based on~\Cref{lemma:constant_and_variable}, 
        and use an MLP in the sixth layer to select $\mathbf{t} \text{ or } \mathbf{0}$ based on whether $t_0 = 0 \, ?$  denoted as $\mathbf{v}$, which  can be implemented using~\Cref{lemma:selection}.
        \item \textbf{Layer 7}. In the seventh layer, we compute $\mathbf{x}[(i+m)\, \bmod \, n] \, \oplus \, (\mathbf{t} >> 1)$ according to~\Cref{lemma:variable_xor_variable}.
        \item \textbf{Layer 8}. In the eighth layer, we compute $\mathbf{x}[(i+m)\, \bmod \, n] \, \oplus \, (\mathbf{t} >> 1) \, \oplus \, \mathbf{v}$ according to~\Cref{lemma:variable_xor_variable}.
    \end{itemize}
    The result of the rotation, $\mathbf{z}$, is obtained in the eighth layer. 
\\
\\
    \textbf{Block 2 (Extraction)}. In the second block, we simulate the extraction process:
    \begin{align*}
        \mathbf{x}[i] & \gets \mathbf{z} \\
        \mathbf{y} & \gets \mathbf{x}[i] \\
        \mathbf{y} & \gets \mathbf{y} \, \oplus \, (\mathbf{y} >> u)\\
        \mathbf{y} & \gets \mathbf{y} \, \oplus \, ((\mathbf{y} << s)\, \land \, \mathbf{b})\\
        \mathbf{y} & \gets \mathbf{y} \, \oplus \, ((\mathbf{y} << t)\, \land \, \mathbf{c})\\
        \mathbf{y} & \gets \mathbf{y} \, \oplus \, (\mathbf{y} >> l)
    \end{align*}
    $u,l,s,t,\mathbf{b},\mathbf{c}$ are constants.
    \begin{itemize}
        \item \textbf{Layer 9 - 10}. We can compute $(\mathbf{y} >> u)$ in the ninth layer and \\ $\mathbf{y} \, \oplus \, (\mathbf{y} >> u)$ in the tenth layer according to~\Cref{corollary:constant_right_left_shift_variable} and~\Cref{lemma:variable_xor_variable}.
        \item \textbf{Layer 11 - 12}. We can compute $(\mathbf{y} << s) \, \land \, \mathbf{b}$ in the eleventh layer and \\ $\mathbf{y} \, \oplus \, ((\mathbf{y} << s) \, \land \, \mathbf{b})$ in the twelfth layer according to~\Cref{lemma:constant_and_variable}, \Cref{corollary:constant_right_left_shift_variable}, and~\Cref{lemma:variable_xor_variable}.
        \item \textbf{Layer 13 - 14}. We can compute $(\mathbf{y} << t) \, \land \, \mathbf{c}$ in the thirteenth layer and $\mathbf{y} \, \oplus \, ((\mathbf{y} << t) \, \land \, \mathbf{c})$ in the fourteenth layer according to~\Cref{lemma:constant_and_variable}, \Cref{corollary:constant_right_left_shift_variable}, and~\Cref{lemma:variable_xor_variable}.
        \item \textbf{Layer 15 - 16}. We can compute $(\mathbf{y} >> l)$ in the fifteenth layer and \\ $\mathbf{y} \, \oplus \, (\mathbf{y} >> l)$ in the sixteenth layer according to~\Cref{corollary:constant_right_left_shift_variable} and~\Cref{lemma:variable_xor_variable}.
    \end{itemize}
    Finally, we generate the $t$-th random number  $\mathbf{y}_t$ in the sixteenth layer.
\\
\\
    \textbf{Layer 17}. In the final layer, we utilize the obtained $judge$ to determine 
    whether the next output token is a generated random number ($\mathbf{y}_t$) 
    or an initial value ($\mathbf{x}_{now_t}$ or $\mathbf{x}_{next_x}$) 
    or right arrow signs `$\Rightarrow$',
    which we represent with function \\ $f(judge,\mathbf{y}_t,\mathbf{x}_{now_t},\mathbf{x}_{next_x},\Rightarrow)$ as follows.
    \begin{equation*}
        f(judge,\mathbf{y}_t,\mathbf{x}_{now_t},\mathbf{x}_{next_x}) = \begin{cases}
            \mathbf{y}_t & \text{if } judge = 0\\
            \Vec{0}(\Rightarrow) & \text{if } judge = n \\
            \mathbf{x}_{now_t} & \text{if } judge = now_t\\
            \mathbf{x}_{next_x} & \text{otherwise}
        \end{cases}
    \end{equation*}
    The function $f$ corresponds to a conditional selection operation 
    and can be implemented by an MLP according to~\Cref{lemma:selection}.
    Note that while~\Cref{lemma:selection} provides a 
    binary selection function, 
    we can implement the multi-class selection function $f$ by 
    composing multiple binary selections. Specifically, 
    we can first select between $\mathbf{y}_t$ and non-$\mathbf{y}_t$ cases based on whether $judge = 0$, 
    then select between $\Rightarrow$ and non-$\Rightarrow$ cases 
    when $judge \neq 0$ based on whether $judge = n$, 
    and finally select between $\mathbf{x}_{now_t}$ and $\mathbf{x}_{next_x}$ 
    based on whether $judge = now_t$. 
    This composition of binary selections allows us to implement 
    the desired multi-class selection function $f$ using the binary selection mechanism from~\Cref{lemma:selection}.
    Finally, we pass the output through a softmax layer to generate the next token $s_{i+1}$.
    Similarly to~\Cref{proof:LCG}, 
    the parameters of the Transformer are polynomially bounded by $O(poly(n))$.
\end{proof}

\end{document}